\documentclass[11pt]{article}
\usepackage{enumerate}
\usepackage{pdfsync}
\usepackage[OT1]{fontenc}

\usepackage[usenames]{color}
\usepackage{smile}
\usepackage[colorlinks,
            linkcolor=red,
            anchorcolor=blue,
            citecolor=blue
            ]{hyperref}
\usepackage{fullpage}
\usepackage[protrusion=true,expansion=true]{microtype}
\usepackage{pbox}
\usepackage{setspace}
\usepackage{tabularx}
\usepackage{float}
\usepackage{wrapfig,lipsum}
\usepackage{enumitem}
\usepackage{microtype}
\usepackage{graphicx}
\usepackage{subfigure}
\usepackage{pgfplots}
\usepackage{mathtools}
\usepackage{caption}
\usetikzlibrary{arrows,shapes,snakes,automata,backgrounds,petri}
\usepackage{booktabs} 

\allowdisplaybreaks
\usepackage{colortbl}

\usepackage{xcolor}

\usepackage{makecell}
\usepackage{pifont}
\newcommand{\cmark}{\ding{51}}%
\newcommand{\xmark}{\ding{55}}%
\newcommand{\DKL}{D_{\mathrm{KL}}}
\DeclareMathOperator{\Ber}{Ber}

\makeatletter
\newcommand*{\rom}[1]{\expandafter\@slowromancap\romannumeral #1@}
\makeatother
\title{\huge Feel-Good Thompson Sampling for Contextual Dueling Bandits}
\author
{
    Xuheng Li\thanks{Department of Computer Science, University of California, Los Angeles, CA 90095, USA; e-mail: {\tt xuheng.li@cs.ucla.edu}}
    ~~~~
    Heyang Zhao\thanks{Department of Computer Science, University of California, Los Angeles, CA 90095, USA; e-mail: {\tt hyzhao@cs.ucla.edu}}
    ~~~~
    Quanquan Gu\thanks{Department of Computer Science, University of California, Los Angeles, CA 90095, USA; e-mail: {\tt qgu@cs.ucla.edu}}
}

\newcommand{\regret}{\mathrm{Regret}}
\newcommand{\BE}{\mathrm{BE}}
\newcommand{\FG}{\mathrm{FG}}
\newcommand{\LS}{\mathrm{LS}}

\newcommand{\rd}{\text{d}}
\newcommand{\dc}{\mathrm{dc}}
\definecolor{LightCyan}{rgb}{0.58, 0.94, 0.85}
\ifdefined\final
\usepackage[disable]{todonotes}
\else
\usepackage[textsize=tiny]{todonotes}
\fi
\begin{document}
    \date{}
    \maketitle

\begin{abstract}
Contextual dueling bandits, where a learner compares two options based on context and receives feedback indicating which was preferred, extends classic dueling bandits by incorporating contextual information for decision-making and preference learning.
Several algorithms based on the upper confidence bound (UCB) have been proposed for linear contextual dueling bandits. However, no algorithm based on posterior sampling has been developed in this setting, despite the empirical success observed in traditional contextual bandits.
In this paper, we propose a Thompson sampling algorithm, named FGTS.CDB, for linear contextual dueling bandits.
At the core of our algorithm is a new Feel-Good exploration term specifically tailored for dueling bandits. This term leverages the independence of the two selected arms, thereby avoiding a cross term in the analysis. We show that our algorithm achieves nearly minimax-optimal regret, i.e.,  $\tilde\cO(d\sqrt T)$, where $d$ is the model dimension and $T$ is the time horizon.
Finally, we evaluate our algorithm on synthetic data and observe that FGTS.CDB outperforms existing algorithms by a large margin.
\end{abstract}

\section{Introduction}

Reinforcement learning from human feedback (RLHF) has become a popular methodology in the alignment of large language models (LLMs, \citealt{ouyang2023training, diao2023lmflow}).
In RLHF, it is often easier for the human user to compare two responses than providing a numerical reward/score based on a common standard.
Therefore, existing works on RLHF~\citep{zhu2023principled, ji2023provable} focus on a model where the learning agent has a dataset of users' preferences among several choices.
The preferences are often assumed to follow the Plackett-Luce (PL) model~\citep{soufiani2014computing, khetan2016data, ren2018pac}, where the probability of the user favoring a certain choice is proportional to the exponential of the reward function, and the special case where two choices are presented to the user is called the Bradley-Terry-Luce (BTL) model~\citep{hunter2004mm, luce2005individual}.
The online version of the preference-based model, called the \textit{dueling bandits}, has been studied extensively~\citep{Yue2012TheKD, Zoghi2014RelativeUC, komiyama2015regret} when the set of the action space is fixed and finite (i.e., the multi-armed dueling bandits).
%
Recently, a more general model, the \textit{(linear) contextual dueling bandit}~\citep{saha2021optimal, bengs2022stochastic}, has been proposed.
This model has important features including time-varying and possibly infinite action spaces, along with a context-dependent reward function with a linear structure, which capture important practical situations.
A number of algorithms have been proposed for (linear) contextual dueling bandits, including MaxInP~\citep{saha2021optimal}, CoLSTIM~\citep{bengs2022stochastic} and VACDB~\citep{di2023variance}, all of which are based on the upper confidence bound (UCB) technique for exploration.

Under the setting of traditional contextual bandits, Thompson sampling~\citep{thompson1933likelihood} is another technique for exploration apart from UCB-based methods, and superior empirical performance has been observed when applying Thompson sampling to various tasks~\citep{chapelle2011empirical, osband2017posterior}.
Instead of deterministically learning a model, in Thomson sampling, models are sampled from a posterior distribution constructed on historic observations.
It has been widely studied in both the multi-armed setting \citep{agrawal2012analysis, kaufmann2012thompson, agrawal2017near, jin2021mots} and the linear setting \citep{agrawal2013thompson}.
Later, \citet{zhang2022feel} showed that the frequentist regret of linear Thompson sampling is suboptimal in the worst case and proposed a new variant of Thompson sampling called \textit{Feel-Good Thomson sampling} (FGTS) to overcome this issue.
The effectiveness of FGTS is theoretically justified:
when applied to linear contextual bandits, FGTS can achieve the minimax-optimal regret bound as UCB-based algorithms like LinUCB~\citep{li2010contextual} or OFUL~\citep{abbasi2011improved}. 

Despite the success of Thompson sampling algorithms in traditional contextual bandits, there have been few works that apply this technique to contextual dueling bandits.
The notable exception is a double Thompson sampling approach proposed by \citet{wu2016double}. However, this approach is limited to multi-armed dueling bandits and cannot be modified for (linear) contextual dueling bandits. 
In addition, it is also unknown whether algorithms based on Thompson sampling can achieve the same minimax-optimal regret bounds as UCB-based algorithms for contextual dueling bandits.
Therefore, we raise the following question:
\begin{center}
\textit{Is it possible to design an efficient algorithm for contextual dueling bandits based on Feel-Good Thompson sampling?}
\end{center}

In this paper, we affirmatively answer this question by solving the problem of linear contextual dueling bandits under the framework of Feel-Good Thompson sampling.
We summarize our contributions as follows:
\begin{itemize}[leftmargin=*]
\item We propose a new algorithm named FGTS.CDB for the problem of linear contextual dueling bandits, which is based on Feel-Good Thompson sampling.
Compared with existing FGTS algorithms for standard contextual dueling bandits~\citep{zhang2022feel}, we introduce a new Feel-Good exploration term designed specially for the comparison of two actions.
Compared with UCB-based approaches, our algorithm can handle the case of large action spaces more efficiently.
\item We prove that our algorithm enjoys a minimax-optimal regret bound of $\tilde\cO(d\sqrt T)$ in expectation, where $d$ is the feature dimensionality and $T$ is the number of rounds.
The new Feel-Good exploration term plays a crucial role in the proof by eliminating cross terms that arise from the comparison of actions.
\item We extend our analysis to the setting of general reward functions, and manage to recover the regret bound for several cases of interest, including the cases of finite action sets and finite model sets.
\item We conduct experiments to compare our algorithms with several strong baselines, including MaxInP, MaxPairUCB~\citep{saha2021optimal}, CoLSTIM~\citep{bengs2022stochastic} and VACDB~\citep{di2023variance}. We observe that the performance of FGTS.CDB is significantly better than all baselines.
\end{itemize}

\noindent\textbf{Notation.}
We use plain case letters to denote scalars and lowercase boldface letters to denote vectors.
We use $\langle\cdot, \cdot\rangle$ to denote the inner product of vectors.
For a vector $\xb$, $\|\xb\|_2$ denotes its $\ell_2$-norm.
We use $[N]$ as a shorthand for the set $\{1, 2, \dots, N\}$.
We use standard asymptotic notations including $\cO(\cdot)$, $\Omega(\cdot)$ and $\Theta(\cdot)$, while $\tilde\cO(\cdot)$ $\tilde\Omega(\cdot)$ and $\tilde\Theta(\cdot)$ hide logarithmic factors.

\newcolumntype{g}{>{\columncolor{LightCyan!40}}c}
\begin{table*}[ht!]
\caption{Comparison of our algorithm, FGTS.CDB, against existing algorithms for linear contextual dueling bandits. FGTS.CDB is the first algorithm for linear contextual dueling bandits using the technique of Thompson sampling. Our algorithm is also the first that can be easily applied to the case of infinite action spaces (modification for MaxInP is more complex). The regret bounds hold for linear contextual dueling bandits of $T$ rounds, with $d$-dimensional feature vectors and the action space of size $K$. } 
\centering\small
\begin{tabular}{gggg}
\toprule
\rowcolor{white} Algorithm & Main technique & Infinite action space? & Regret\\
\midrule
\rowcolor{white} \makecell{MaxInP\\\small{\citep{saha2021optimal}}} & UCB + Adaptive Selection  & \cmark & $\tilde\cO(d\sqrt T)$\\
\rowcolor{white} \makecell{CoLSTIM\\\small{\citep{bengs2022stochastic}}} & Perturbed UCB & \xmark & $\tilde\cO(d\sqrt T)$\\
\rowcolor{white} \makecell{Sta'D\\\small{\citep{saha2021optimal}}} & SupLinUCB + Adaptive Selection & \xmark & $\tilde\cO(\sqrt{dT\log K})$\\
\rowcolor{white} \makecell{SupCoLSTIM\\\small{\citep{bengs2022stochastic}}} & Perturbed SupLinUCB & \xmark & $\tilde\cO(\sqrt{dT\log K})$\\
\rowcolor{LightCyan!40}\makecell{~~~~~~~FGTS.CDB~~~~~~~\\\small{(This work)}} & Feel-Good Thompson Sampling & \cmark & $\tilde\cO(d\sqrt T)$\\
\bottomrule
\end{tabular}
\end{table*}

\section{Related Work}

\noindent\textbf{Dueling bandits.}
First proposed by \citet{Yue2012TheKD}, the dueling bandit problem involves a learner sequentially selecting a pair of arms among multiple choices based on the noisy binary observations revealing the relative preference of the chosen arms. Under their multi-armed dueling bandit setting, \citet{Zoghi2014RelativeUC} proposed RUCB, a UCB-based algorithm which achieves an $\cO(K \log T / \Delta)$ regret upper bound where $K$ is the number of arms, $T$ is the number of rounds, and $\Delta$ stands for the gap between the best arm and the second-best arm. Later, \citet{komiyama2015regret} proposed RMED with a more sophisticated arm selection phase whose regret matches the lower bound with optimal constant. Relaxing the typical Condorcet winner setting where it is assumed that there is one arm that beats all the other arms, researchers also investigated other variants of multi-armed dueling bandits which assumed the existence of Copeland Winner \citep{Zoghi2015CopelandDB, wu2016double, komiyama2016copeland}, Borda winner \citep{jamieson2015sparse, falahatgar2017maximum, Heckel2016ActiveRF, Saha2021AdversarialDB, brandt2022finding, wu2023borda}, or von Neumann winner \citep{Dudk2015ContextualDB, balsubramani2016instance, Ramamohan2016DuelingBB}. 

\noindent\textbf{Contextual dueling bandits.}~There is also a large body of literature on contextual dueling bandits, where dueling bandits with contextual information is considered \citep{kumagai2017regret, saha2021optimal,saha2022efficient, bengs2022stochastic, di2023variance}. \citet{kumagai2017regret} studied dueling bandits with a cost function over a continuous space and achieved a dimension-free regret under the strong convexity and smoothness assumption. \citet{saha2021optimal} considered contextual dueling bandits with generalized linear classes and proposed an algorithm MaxInP with an $\tilde{\cO}(d\sqrt{T})$ regret and Sta'D with an $\tilde{\cO}(\sqrt{dT\log K})$ regret. \citet{bengs2022stochastic} proposed CoLSTIM and further extended it to the contextual linear stochastic transitivity model. Recently, \citet{di2023variance} proposed an action-elimination based algorithm VACDB, with a tighter variance-dependent regret bound. It is worth mentioning that all the existing algorithms for contextual dueling bandits need to either maintain a subset of eligible arms or maximize the randomly perturbed rewards over all the possible arms, which are only applicable to finite action space. 

\noindent\textbf{Feel-Good Thompson sampling (FGTS).} FGTS was proposed by \citet{zhang2022feel} to fill the gap between the practical effectiveness of Thompson sampling and a lack of frequentist-type regret guarantee.
When applied to linear contextual bandits, FGTS achieves a regret bound of $\tilde\cO(d\sqrt T)$ that matches the lower bound of $\Omega(d\sqrt T)$.
The analysis of this algorithm is based on the decoupling of arm selection with model parameters.
\citet{fan2023the} proposed a unified framework for the analysis of FGTS applied to a number of variants of linear contextual bandits.
Another line of works extends the idea of FGTS to reinforcement learning, including Model-based Optimistic Posterior Sampling (MOPS) for Markov decision processes~\citep{agarwal2022model} and conditional Posterior Sampling with Booster for two-player Markov games~\citep{xiong2022self}.
Our work is the first attempt to apply FGTS to contextual dueling bandits.

\noindent\textbf{Sampling-based algorithms for dueling bandits.}
\citet{wu2016double} proposed a double Thompson sampling algorithm for multi-arm dueling bandits which achieves a regret bound of $\cO(K^2\log T)$ where $K$ is the number of arms.
\citet{sui2017multi} also proposed an algorithm based on Thompson sampling that converted multi-dueling bandits to standard bandits.
However, these two algorithms cannot be modified for the setting of (linear) contextual dueling bandits because they depend on the count of comparison outcomes between the arms and $T$ is the number of rounds.
In contextual dueling bandits, this is infeasible because the set of arms are different across different rounds.
Other algorithms are based on the sampling of policies rather than model parameters.
For example, \citet{xiong2023gibbs} proposed a KL-constrained framework which uses Gibbs sampling.
Nonetheless, this work focuses on fine-tuning LLMs, and the regret studied in this work has an additional term that measures the difference between the learned policy and the original policy.
\citet{novoseller2020dueling} studied the application of posterior sampling in preference-based reinforcement learning.
However, the regret bound of the algorithm relies on the assumption of finite state and action sets and cannot be trivially extended to linear contextual dueling bandits.

\section{Problem Setting} \label{sec:setup}

In this work, we study the setting where the agent repeatedly interact with the agent to receive prompts and query preferences between the two chosen responses.

\noindent\textbf{Linear contextual dueling bandits.}
We focus on the setting of contextual dueling bandits with contextual information embodied in both the prompt and the action space, similar to~\citet{zhang2022feel}.
Let $\cX$ be the set of prompts and $\cA$ be the set of all possible responses.
During round $t$ in a total of $T$ rounds, the agent receives a prompt $x_t\in\cX$, along with a corresponding action space $\cA_t\subset\cA$ which can both vary across different rounds.
The agent then selects two responses (more commonly referred to as arms in the bandit context) $a_t^1, a_t^2\in\cA_t$ and receives a randomized preference $y_t$ whose distribution depends on an underlying reward function $r_*: \cX\times\cA\to\RR$.
$y_t=1$ represents the case where $a_t^1$ is preferred over $a_t^2$, and $y_t=-1$ otherwise.
We assume that the reward function class adopts a linear structure:
\begin{assumption}[Linear reward]\label{assumption:linear}
We assume that the reward function is parameterized by $r_{\btheta}=\langle\btheta, \bphi(s, a)\rangle$ for some known feature mapping $\bphi:\cX\times\cA\to\RR^d$. Specifically, the real value function is $r_*(x, a)=\langle\btheta_*, \bphi(x, a)\rangle$ for some vector $\btheta_*\in\RR^d$ hidden from the learning agent. We assume that $\|\bphi(s, a)\|_2\le1$ for all $(x, a)\in\cX\times\cA$, and $\|\btheta\|_2\le B$. Thus, the reward function is bounded by $|r_{\btheta}(\cdot, \cdot)|\le B$.
\end{assumption}

The setting we study is equivalent to those of previous works on contextual dueling bandits.
\citet{saha2021optimal} and \citet{bengs2022stochastic} considered a time-varying action space $\cS_t=\{\xb_1, \dots, \xb_K\}$ where each action is represented by a $d$-dimensional vector called the \textit{contextual vector}, and the reward function is defined as $r_*(\ab)=\langle\btheta_*, \ab\rangle$.
The contextual vector depends on both the prompt $x_t$ and the response $a_t$, and can be viewed as a counterpart of the feature mapping $\bphi(\cdot, \cdot)$ in Assumption \ref{assumption:linear}.
Compared with the contextual vector, our formulation is more general when considering other types of function approximations (see Section \ref{section:nonlinear_main}).

\noindent\textbf{Stochastic preference model.}
In this work, we assume that the preference $y_t$ follows a Bernoulli distribution according to the Bradley-Terry-Luce (BTL) model~\citep{hunter2004mm, luce2005individual}:
Given context $x_t$ and responses $a_t^1, a_t^2$, the probability of $a_t^1$ is preferred over $a_t^2$ is
\begin{align*}
\PP(y_t=1|x_t, a_t^1, a_t^2)&=\frac{\exp(r_*(x_t, a_t^1))}{\exp(r_*(x_t, a_t^1))+\exp(r_*(x_t, a_t^2))}=\exp(-\sigma(r_*(x_t, a_t^1)-r_*(x_t, a_t^2))),
\end{align*}
where $\sigma(z)=\log(1+\exp(-z))$.

Some other works study a more general setting called the Plackett-Luce (PL) model~\citep{soufiani2014computing, khetan2016data, ren2018pac}, where the learning agent selects $q\ge2$ arms $a_t^1, \dots, a_t^q\in\cA_t$ in round $t$ and receives the preference $o_t\in[q]$.
The probability of $a_t^j$ being preferred is
\[
\PP(o_t=j)=\frac{\exp(r_*(x_t, a_t^j))}{\sum_{j'=1}^q\exp(r_*(x_t, a_t^{j'}))}.
\]
The BTL model can be seen as a special case of the PL model by fixing $q=2$.
\citet{saha2021optimal} showed that under the PL model, the worst-case regret of any algorithm for dueling bandits is $\Omega(d\sqrt T)$, regardless of the choice of $q$.
Therefore, provided that the learner is permitted to select any number of arms, it suffices to design a minimax-optimal algorithm where two arms are selected in each round, which is shown to be true for our algorithm in Section~\ref{section:MainResults}.

\noindent\textbf{Learning Objective.}
Our goal is to minimize the cumulative average regret defined as
\[
\regret(T)\coloneqq\sum_{t=1}^T\bigg[r_*(x_t, a_t^*)-\frac{r_*(x_t, a_t^1)+r_*(x_t, a_t^2)}2\bigg],
\]
where $a_t^*=\argmax_{a\in\cA}r_*(x_t, a)$ is the optimal response at time $t$.
The regret we study is exactly the same as the regret studied in \citet{saha2021optimal} and \citet{bengs2022stochastic}.
The regret is also equivalent to the dueling bandit regret studied in~\citet{Yue2012TheKD}, defined as
\begin{align*}
\sum_{t=1}^T&\frac12\Big[\Big(\exp(-\sigma(r_*(x_t, a_t^*)-r_*(x_t, a_t^1)))-\frac12\Big)+\Big(\exp(-\sigma(r_*(x_t, a_t^*)-r_*(x_t, a_t^2)))-\frac12\Big)\Big],
\end{align*}
because $\exp(-\sigma(z))-1/2=\Theta(z)$ for $z\in[-2B, 2B]$.

\section{Algorithm Description}

We now present our algorithm, named FGTS.CDB, for linear contextual dueling bandits. The pseudocode is shown in Algorithm \ref{algorithm:FGTS.CDB}.

\begin{algorithm}[ht!]
\caption{FGTS.CDB}\label{algorithm:FGTS.CDB}
\begin{algorithmic}[1]
\STATE Given hyperparameters $\eta, \mu$. Initialize $S_0=\varnothing$.
\FOR{$t=1, \dots, T$}
\STATE Receive prompt $x_t$ and action space $\cA_t$.
\FOR{$j=1, 2$}
\STATE Sample model parameter $\btheta_t^j$ from the posterior distribution $p^j(\cdot|S_{t-1})$, defined in \eqref{eq:def_posterior}.
\STATE Select response $a_t^j=\argmax_{a\in\cA_t}\langle\btheta_t^j, \bphi(x_t, a)\rangle$.
\ENDFOR
\STATE Receive preference $y_t$.
\STATE Update dataset $S_t\gets S_{t-1}\cup\{(x_t, a_t^1, a_t^2, y_t)\}$.
\ENDFOR
\end{algorithmic}
\end{algorithm}

In our algorithm, the agent first samples model parameters $\btheta_t^1$ and $\btheta_t^2$ independently, following posterior distributions $p^1(\cdot|S_{t-1})$ and $p^2(\cdot|S_{t-1})$, respectively.
The posterior distributions are defined as
\begin{align}\label{eq:def_posterior}
p^j(\btheta|S_{t-1})\propto\exp\bigg(-\sum_{i=1}^{t-1}L^j(\btheta, x_i, a_i^1, a_i^2, y_i)\bigg)p_0(\btheta),
\end{align}
where $L^j$ is the likelihood function, and $p_0(\cdot)$ is the prior distribution.
Sampling from such a posterior distribution can be implemented via Langevin Monte Carlo (LMC), which has been studied extensively in the literature~\citep{roberts1996exponential, bakry2014analysis}.
Afterwards, actions $a_t^j$ are selected to maximize the inner product of the parameter $\btheta_t^j$ and the feature mapping $\bphi(x_t, a_t^j)$ for $j=1, 2$.
Finally, the agent receives the binary preference $y_t\in\{\pm1\}$ and augments the dataset with $(x_t, a_t^1, a_t^2, y_t)$.

\noindent\textbf{Feel-Good Thompson sampling.}
In our algorithm, the likelihood function is defined as
\begin{align*}
L^j(\btheta, x, a^1, a^2, y)&=\eta\sigma(y\langle\btheta, \bphi(x, a^1)-\bphi(x, a^2)\rangle)-\mu\max_{a'\in\cA}\langle\btheta, \bphi(x, a')-\bphi(x, a^{3-j})\rangle,
\end{align*}
where $\eta$ and $\mu$ are hyperparameters.
In the definition above, the first term can be treated as the log-likelihood function on the observation $(x, a^1, a^2, y)$; the second term encourages exploration of $\btheta$ with large reward in previous rounds, which is referred to as \textit{Feel-Good exploration} in the literature~\citep{zhang2022feel}.
Without the Feel-Good exploration term, i.e., when $\mu=0$, $L^j$ reduces to the likelihood function used in standard Thompson sampling algorithms.

\noindent\textbf{Comparison with FGTS for traditional contextual bandits.}
The differences between FGTS.CDB and existing FGTS algorithms for traditional contextual bandits~\citep{zhang2022feel} are twofold.
Firstly, due to the preferential feedback, the least-squares term in previous algorithms is naturally replaced with a term in the form of logistic regression.
The more important difference lies in the Feel-Good exploration terms.
In our Feel-Good exploration term, there is an additional inner product of the current model parameter $\btheta$ and the feature vector of the adversarial arm $\bphi(x, a^{3-j})$.
This additional term is a better design for the setting of contextual dueling bandits because the affecting factor of the observation $y_t$ is the difference between the rewards of two arms rather than the reward of a single arm.
Additionally, this term plays a crucial role in the proof, as we will show in Section \ref{section:proof_highlights}.

\noindent\textbf{Comparison with UCB-based algorithms.}
We highlight that besides the model learning technique, there is also a stark difference in the arm selection scheme between UCB-based algorithms (including MaxInP~\citep{saha2021optimal}, CoLSTIM~\citep{bengs2022stochastic}, VACDB~\citep{di2023variance}) and FGTS.CDB.
In UCB-based algorithms, arms are often selected based on a bonus term in the form of $\|\bphi(x_t, a_t^1)-\bphi(x_t, a_t^2)\|_{\bSigma^{-1}}$, for some positive definite matrix $\bSigma$, which encourages the selection of more separated arms.
Thus, the bonus term is essential for exploration in UCB-based algorithms, and results in the dependence of the selected arms.
In FGTS.UCB, however, the arms $a_t^1$ and $a_t^2$ are just maximizers of the learned reward function, and are independent conditioned on the history $S_{t-1}$.
This is possible because exploration is accomplished by Thompson sampling in our algorithm, and the bonus term becomes unnecessary.
In addition, the independence of arms $a_t^1$ and $a_t^2$ (conditioned on $S_{t-1}$) is a crucial property in our proof as we will show in Section \ref{section:proof_highlights}.

From the viewpoint of computational complexity, the arm selection scheme of FGTS.CDB is also superior to those of existing algorithms.
When the action space $\cA_t$ is infinite and continuous, the arm selection phase of FGTS.CDB can still be implemented by solving an optimization problem.
In contrast, MaxInP calculates a set of promising arms in each round, which causes additional computational overhead in the case of infinite action spaces.
CoLSTIM needs to take the maximum of randomly perturbed rewards corresponding to each contextual vector, which is infeasible when the number of arms is infinite.
Therefore, FGTS.CDB is the first algorithm that can be easily applied to the setting of infinite action spaces.

\section{Main Results}\label{section:MainResults}

In this section, we present the regret bounds of Algorithm \ref{algorithm:FGTS.CDB}, which is minimax-optimal. We first introduce the following assumption about the prior distribution $p_0$:
\begin{assumption}\label{assumption:lipschitz}
The logarithm of the prior distribution is $L$-Lipschitz, i.e., for all $\btheta_1, \btheta_2\in\{\btheta:\|\btheta\|_2\le B\}$, we have
\begin{align*}
|\log p_0(\btheta_1)-\log p_0(\btheta_2)|\le L\|\btheta_1-\btheta_2\|_2.
\end{align*}
\end{assumption}
Assumption \ref{assumption:lipschitz} is satisfied for most commonly-used prior distributions, including the uniform distribution ($L=0$) and the Gaussian distribution $\cN(0, \sigma_0^2\Ib_d)$ restricted to the ball $\{\btheta:\|\btheta\|_2\le B\}$ ($L=B/\sigma_0^2$).

We now present the main theorem:
\begin{theorem}\label{theorem:main}
Under Assumptions \ref{assumption:linear} and \ref{assumption:lipschitz}, assume that the hyperparameters are selected as $\eta=0.25$ and $\mu=1/(10e^B\sqrt T)$, then the expectation of the regret of Algorithm \ref{algorithm:FGTS.CDB} satisfies
\[
\EE[\regret(T)]=\tilde\cO(d\sqrt T).
\]
\end{theorem}

The following theorem provides a regret lower bound for contextual dueling bandits:
\begin{theorem}\label{theorem:lower_bound}
Under Assumption~\ref{assumption:linear}, assume that for all $t\in[T]$, we have $\{\bphi(x_t, a):a\in\cA_t\}=\{\ub:\|\ub\|_2\le1\}$. We also assume that $T\ge\max\{72B^{-2}d^2, d\}$. Then for any algorithm for linear contextual dueling bandits, there exists $\btheta_*$ such that the expectation of the regret satisfies
\begin{align*}
\EE[\regret(T)]=\Omega(d\sqrt T).
\end{align*}
\end{theorem}

\begin{remark}
Theorem~\ref{theorem:lower_bound} shows that the regret lower bound of any algorithm for linear contextual dueling bandits is $\Omega(d\sqrt T)$.
Note that Theorem 3.1 in~\citet{bengs2022stochastic} also provides a regret lower bound of algorithms for linear contextual dueling bandits, but their lower bound is looser than ours by a factor of $\sqrt d$.
They applied the analysis for the case of bounded $\ell_\infty$-norm to the case of bounded $\ell_2$-norm, which yields loose inequalities~\citep{lattimore2020bandit}.
Combining Theorem~\ref{theorem:main} and Theorem~\ref{theorem:lower_bound}, we conclude that the regret bound of FGTS.CDB matches the worst-case regret.
Our regret bound also matches that of UCB-based algorithms including MaxInP~\citep{saha2021optimal} and CoLSTIM~\citep{bengs2022stochastic}.
More recently, \citet{di2023variance} proposed an algorithm that has a variance-dependent regret bound, and it is an interesting future direction to design a variance-aware sampling-based algorithms for contextual dueling bandits.
\end{remark}
\begin{remark}
Some other works assume that the size of the action space $K$ is finite and derive algorithms with the regret bound $\tilde\cO(\sqrt{dT\log K})$, including Sta'D~\citep{saha2021optimal} and SupCoLSTIM~\citep{bengs2022stochastic}.
We first note that the regret bound of $\tilde\cO(\sqrt{dT\log K})$ is not a contradiction against Theorem~\ref{theorem:main} due to the assumption of a total of $K$ arms.
More specifically, the proof of Theorem~\ref{theorem:lower_bound} involves contructing $\{\bphi(x_t, a):a\in\cA_t\}=\{\ub:\|\ub\|_2\le1\}$, so $K=2^d$, and $\tilde\cO(\sqrt{dT\log K})=\tilde\cO(d\sqrt T)$.
When $K$ is exponential in the model dimensionality $d$, which is more often the case in the setting of contextual dueling bandits, the regret bound of FGTS.CDB is at least as good as these algorithms.
In addition, our algorithm is more computationally efficient because it does not need to perform arm elimination or to apply random pertubations to each arm.
\end{remark}

\section{Extension to Nonlinear Reward}\label{section:nonlinear_main}

In this section, we relax the assumption of linear reward functions.
Instead, we make the following assumption about the reward function class:
\begin{assumption}\label{assumption:bounded}
The parameter space $\Theta$ is a measurable space with measure $\bar\mu$ and metric $d$. The model is well-specified, i.e., $\theta_*\in\Theta$. The reward function is uniformly bounded by $B$ and is $L_0$-Lipschitz in $\theta$.
\end{assumption}
We define a shorthand notation
\[
\Delta r_\theta(x, a^1, a^2)\coloneqq r_\theta(x, a^1)-r_\theta(x, a^2).
\]
In order to characterize the complexity of the reward function class, similar to~\citet{zhang2022feel}, we define the decoupling coefficient $\dc$ to be such that for any $\lambda>0$ and any joint distribution $P$ over $\Theta\times\cA\times\cA$, we have
\begin{align*}
&\EE_{(\theta, a^1, a^2)\sim P}[\Delta r_{\theta}(x, a^1, a^2)-\Delta r_{\theta_*}(x, a^1, a^2)]\\
&\le\lambda\dc+\frac{1}{4\lambda}\EE_{\tilde\theta\sim P|_\theta}\EE_{(a^1, a^2)\sim P|_{(a^1, a^2)}}[\Delta r_{\tilde\theta}(x, a^1, a^2)-\Delta r_{\theta_*}(x, a^1, a^2)]^2,
\end{align*}
where $P|_\theta$ and $P|_{(a^1, a^2)}$ are the marginal distributions of $P$.
We have shown in Lemma \ref{lemma:decoupling} that in the setting of linear reward functions, the decoupling coefficient is $d$.
Furthermore, if the size of the action space is $K$, then the decoupling coefficient is bounded by $K(K-1)/2=\cO(K^2)$, shown by Lemma 2 in~\citet{zhang2022feel}.

We now introduce modifications to FGTS.CDB for the nonlinear reward.
In accordance with the change of the reward function's structure, the arm selection scheme and the likelihood function are changed to
\begin{align*}
a_j^t&=\max_{a\in\cA_t}r_{\theta_t^j}(x_t, a),\\
L^j(\theta, x, a^1, a^2, y)&=\eta\sigma(y\Delta r_\theta(x, a^1, a^2))-\mu\max_{a'\in\cA}\Delta r_\theta(x, a', a^{3-j}).
\end{align*}
The following theorem characterizes the regret bound of FGTS.CDB with the aforementioned modifications:
\begin{theorem}\label{theorem:nonlinear}
Suppose that Assumptions \ref{assumption:lipschitz} and \ref{assumption:bounded} hold, and the hyperparameters are chosen as $\eta=0.25$ and
\[
\mu=\frac{e^{-B}}{5}\sqrt{\frac{-\log (p_0(\theta_*)\bar\mu(\{\theta:d(\theta, \theta_*)\le2/(L_0T)\})}{T\dc}}.
\]
Then the regret of FGTS.CDB satisfies
\begin{align*}
&\EE[\regret(T)]=\cO(\sqrt{-\log(p_0(\theta_*)\bar\mu(\{\theta:d(\theta, \theta_*)\le2/(L_0T)\})\cdot T\dc}).
\end{align*}
\end{theorem}
\begin{remark}
Theorem \ref{theorem:nonlinear} can be reduced to Theorem \ref{theorem:main} by noting that $-\log(p_0(\theta_*)\bar\mu(\{\theta:d(\theta, \theta_*)\le2/(L_0T)\}))=\tilde\cO(d)$, and $\dc=\cO(d)$.
Furthermore, if we assume that the size of the action space is bounded by $K$, then the regret bound is $\tilde\cO(K\sqrt{dT})$.
This regret bound is minimax-optimal if $K$ is treated as a constant, although the regret bounds of existing algorithms, e.g.,  Sta'D~\citep{saha2021optimal} and SupCoLSTIM~\citep{bengs2022stochastic}, scale with $\sqrt{\log K}$.
\end{remark}
\begin{remark}
Another case of interest is that the reward function has a linear structure, but $\Theta$ is a finite set of size $N$. In this case, if we choose the prior $p_0$ to be the uniform distribution on $\Theta$, then $-\log p_0(\theta_*)=\log N$, and $-\log(\bar\mu(\{\theta:d(\theta, \theta_*)\le2/L_0T\}))=0$ when $T$ is large enough. Therefore, the regret bound is $\cO(\sqrt{dT\log N})$.
\end{remark}





\section{Overview of Proof}\label{section:proof_highlights}

In this section, we present the key techniques in the proof of Theorem \ref{theorem:main}, and details are given in Appendix \ref{section:proof_linear}.
The proof of Theorem \ref{theorem:nonlinear} is similar and is given in Appendix \ref{section:proof_nonlinear}.
In Subsection \ref{subsection:RegretDecomp}, we first introduce a special regret decomposition scheme corresponding to our algorithm and discuss the advantage of our Feel-Good exploration term.
Then, in Subsection \ref{subsection:Potential}, we get into details of the analysis based on the difference of potentials between steps.

\subsection{Regret Decomposition}\label{subsection:RegretDecomp}

The proof for standard Thompson sampling~\citep{zhang2022feel} performs the following regret decomposition:
\begin{align*}
r_*(x_t, a_t^*)-r_*(x_t, a_t)=\underbrace{[r_\theta(x_t, a_t)-r_*(x_t, a_t)]}_{\text{Bellman Error}}-\underbrace{[\max_ar_\theta(x_t, a)-r_*(x_t, a_t^*)]}_{\text{Feel-Good Exploration}},
\end{align*}
where the Bellman Error term refers to the estimation error of $\theta$ evaluated on historic data, and the Feel-Good exploration term refers to the difference between the maximum reward corresponding to $\theta$ and $\theta_*$.
The Bellman Error term can be bounded using the decoupling technique which converts the joint expectation of the model sampling and trajectory into independent expectations.
The Feel-Good Exploration term adopts a crucial structure that does not explicitly contain $a_t$, so the decoupling trivially applies to the Feel-Good Exploration term.
However, the following two challenges arise when studying contextual dueling bandits due to the different definition of the regret: 1. How we can perform the regret decomposition, and 2. whether the Feel-Good Exploration term can also be decoupled.

\noindent\textbf{Challenge 1: Regret decomposition.}
The starting point of deriving a new regret decomposition is the Bellman Error term, which  should correspond to the estimation error on historic data.
As the likelihood in the posterior distribution in~\eqref{eq:def_posterior} is the inner product between $\btheta$ and the difference between the arms, the Bellman Error term is
\begin{align*}
\BE_t^j&=\langle\btheta_t^j-\btheta_*, \bphi(x_t, a_t^j)-\bphi(x_t, a_t^{3-j})\rangle.
\end{align*}
The remaining part of the regret is the Feel-Good Exploration term, which is
\begin{align*}
&\frac{\BE_t^1+\BE_t^2}{2}-\bigg[r_*(x_t, a_t^*)-\frac{r_*(x_t, a_t^1)+r_*(x_t, a_t^2)}{2}\bigg]\\
&=\frac12\Big[\max_{a\in\cA_t}\langle\btheta_t^1, \bphi(x_t, a)-\bphi(x_t, a_t^2)\rangle-\langle\btheta_*, \bphi(x_t, a_t^*)-\bphi(x_t, a_t^2)\rangle\Big]\\
&\qquad+\frac12\Big[\max_{a\in\cA_t}\langle\btheta_t^2, \bphi(x_t, a)-\bphi(x_t, a_t^1)\rangle-\langle\btheta_*, \bphi(x_t, a_t^*)-\bphi(x_t, a_t^1)\rangle\Big].
\end{align*}
Therefore, we set the Feel-Good Exploration term to be
\begin{align*}
\FG_t^j(\btheta)&=\max_{a\in\cA_t}\langle\btheta, \bphi(x_t, a)-\bphi(x_t, a_t^{3-j})\rangle-\langle\btheta_*, \bphi(x_t, a_t^*)-\bphi(x_t, a_t^{3-j})\rangle,
\end{align*}
and the regret decomposition is
\begin{align*}
&r_*(x_t, a_t^*)-\frac{r_*(x_t, a_t^1)+r_*(x_t, a_t^2)}2=\frac12[\BE_t^1+\BE_t^2-\FG_t^1(\btheta_t^1)-\FG_t^2(\btheta_t^2)],
\end{align*}
The regret decomposition inspires the design of the Feel-Good exploration term in the posterior distribution.
Without the additional term $\langle\btheta_t^j, \bphi(x_t, a_t^{3-j})\rangle$ in the likelihood function, the decomposition would contain additional cross terms $\langle\btheta_t^j-\btheta_*, \bphi(x_t, a_t^{3-j})\rangle$ that are hard to analyze.

\noindent\textbf{Challenge 2: Decoupling for both the Bellman Error and the Feel-Good Exploration term.}
We bound $\BE_t^j$ using a decoupling argument that is standard in the literature \citep{zhang2022feel, fan2023the}:
\begin{align}\label{eq:decoupling_main}
\EE[\BE_t^j]\le d\lambda+\frac1{4\lambda}\EE_{S_{t-1}, x_t, a_t^1, a_t^2}\Big[\EE_{\tilde\btheta\sim p^j(\cdot|S_{t-1})}\LS_t(\tilde\btheta)\Big],
\end{align}
where the inequality holds for any constant $\lambda>0$, and the least square term is defined as
\[
\LS_t(\btheta)=\langle\btheta-\btheta_*, \bphi(x_t, a_t^1)-\bphi(x_t, a_t^2)\rangle^2.
\]
Details of this inequality are given in Lemma \ref{lemma:decoupling}.

However, additional challenges arise in the analysis of the Feel-Good exploration term $\FG_t^j(\btheta_t^j)$ because it is different from its counterpart in standard contextual bandits, and the property that the Feel-Good Exploration term does not explicitly contain the trajectory no longer holds.

To counter this challange, we make the following two crucial observations:
\begin{itemize}[leftmargin=*]
\item[1.] The Feel-Good exploration term $\FG_t^j(\btheta_t^j)$ only contains parameter $\btheta_t^j$ and the adversarial arm $a_t^{3-j}$, while $a_t^j$ does not appear in its definition.
\item[2.] Conditioned on the history $S_{t-1}$, $\btheta_t^j$ and $a_t^{3-j}$ are independent because $a_t^{3-j}$ is determined by $\btheta_t^{3-j}$, which is independent with $\btheta_t^j$ according to Algorithm \ref{algorithm:FGTS.CDB}.
\end{itemize}
Therefore, $\btheta_t^j$ can also be decoupled with $a_t^j$:
\begin{align}\label{eq:FG_equality}
\EE[\FG_t^j(\btheta_t^j)]=\EE_{S_{t-1}, x_t, a_t^1, a_t^2}[\EE_{\tilde\btheta\sim p^j(\cdot|S_{t-1})}\FG_t^j(\tilde\btheta)].
\end{align}
Combining \eqref{eq:decoupling_main} and \eqref{eq:FG_equality}, we have
\begin{align}
&\EE[\BE_t^j-\FG_t^j(\btheta)]\le d\lambda+\EE_{S_{t-1}, x_t, a_t^1, a_t^2}\EE_{\tilde\btheta\sim p^j(\cdot|S_{t-1})}\Big[\frac{\LS_t(\tilde\btheta)}{4\lambda}-\FG_t^j(\tilde\btheta)\Big].\label{eq:suffices_to_bound}
\end{align}
To further bound \eqref{eq:suffices_to_bound}, we use a technique based on the analysis of the potential which is shown in Subsection \ref{subsection:Potential}.

\subsection{Potential-Based Analysis}\label{subsection:Potential}

Following \citet{zhang2022feel}, we consider a potential $Z_t^j$ for each round, defined as
\begin{align}\label{eq:def_Zt}
Z_t^j\coloneqq\EE_{S_t}\log\EE_{\tilde\btheta\sim p_0}W_t^j(\tilde\btheta|S_t),
\end{align}
where
\begin{gather}
W^j_t(\btheta|S_t)\coloneqq\exp\bigg(-\sum_{i=1}^t\Delta L^j(\btheta, x_i, a_i^1, a_i^2, y_i)\bigg),\label{eq:def_Wt}\\
\Delta L^j(\btheta, x, a^1, a^2, y)\coloneqq L^j(\btheta, x, a^1, a^2, y)-L^j(\btheta_*, x, a^1, a^2, y).\label{eq:def_DeltaLt}
\end{gather}
Then the posterior distribution satisfies
\begin{align}\label{eq:link_posterior_Wt}
p^j(\tilde\btheta|S_{t-1})=\frac{p_0(\tilde\btheta)W_{t-1}^j(\tilde\btheta|S_{t-1})}{\EE_{\tilde\btheta\sim p_0}W_{t-1}^j(\tilde\btheta|S_{t-1})}.
\end{align}
Using this expression, we can then obtain
\begin{align*}
&Z_t^j-Z_{t-1}^j=\EE_{S_t}\log\EE_{\tilde\btheta\sim p^j(\cdot|S_{t-1})}\exp(-\Delta L^j(\tilde\btheta, x_t, a_t^1, a_t^2, y_t)).
\end{align*}
Based on the above property and the design of $L^j$, we can establish the connection between \eqref{eq:suffices_to_bound} and the potential difference as follows:
\begin{lemma}\label{lemma:LS-FG_main}
Let $Z_t$ be defined in \eqref{eq:def_Zt}, and suppose $\eta\le1/2$. Then we have
\begin{align*}
&\EE_{S_{t-1}, x_t, a_t^1, a_t^2}\EE_{\tilde\btheta\sim p^j(\cdot|S_{t-1})}\Big[\frac{e^{-2B}\eta}{18\mu}\LS_t(\tilde\btheta)-\FG_t^j(\tilde\btheta)\Big]\le\mu^{-1}(Z_{t-1}^j-Z_t^j)+32\mu B^2,
\end{align*}
\end{lemma}
The proof of Lemma \ref{lemma:LS-FG_main} is detailed in Appendix \ref{subsection:bound_LS-FG}.
Inspired by Lemma \ref{lemma:LS-FG_main}, we can choose $\lambda=(9\mu e^{2B})/(2\eta)$ in \eqref{eq:suffices_to_bound}.
Taking the sum over $t$, noting that $Z_0^j=0$, through telescope sum we obtain that
\[
\sum_{t=1}^T\EE[\BE_t^j-\FG_t^j(\btheta_t^j)]\le\frac{-Z_T^j}\mu+\mu T\bigg(\frac{9de^{2B}}{2\eta}+32B^2\bigg).
\]
It then suffices to derive an upper bound for $-Z_T^j$, which is characterized by the following lemma:
\begin{lemma}\label{lemma:ZT_bound_main}
Let $Z_T^j$ be defined by \eqref{eq:def_Zt}, then we have
\begin{align*}
-Z_T^j=\tilde\cO(d),\quad j=1, 2.
\end{align*}
\end{lemma}
Lemma \ref{lemma:ZT_bound_main} shows that the upper bound of $-Z_t^j$ only contains logarithmic factors of $T$ despite the sum of $T$ terms in its definition.
The proof of Lemma \ref{lemma:ZT_bound_main} is shown in Appendix \ref{subsection:ZT_bound} and uses a technique similar to that of Section 5.2 in \citet{zhang2022feel}.
Finally, by selecting $\eta=\Theta(1)$ and $\mu=\Theta(1/\sqrt T)$, we can obtain the regret bound of $\tilde\cO(d\sqrt T)$.

\section{Experiments}


In this section, we investigate the performance of FGTS.CDB through simulation, comparing it with other efficient algorithms proposed for contextual dueling bandits. For each experiment, we run $T = 2500$ rounds. The underlying unknown parameter $\btheta^*$ is randomly generated and normalized to a unit vector. The dimension of feature vectors is set to $d = 5, 10, 15$. We generate a total of $|\cA_t| =32$ distinct arms with feature vectors randomly chosen from $\{-1, 1\}^d$ following the uniform distribution. In every round, given the arm pair selected by the algorithm, a response is generated following the random process described in Section \ref{sec:setup}. Each experiment comprises 10 independent runs. We report and plot the average cumulative regret in Figure \ref{figure:1}, along with the standard deviation shown in the shaded region. For simplicity, we choose the logistic function $\sigma(\cdot)$ as the link function. 

\subsection{Implementation Details of Different Algorithms}

\noindent\textbf{MaxInP.}
The maximum informative pair method introduced by \citet{saha2021optimal} maintains an active set of potential optimal arms in each round. Pairs are selected based on maximizing the uncertainty in the difference between the two arms. Instead of incorporating a warm-up period $\tau_0$ as part of their definitions, we initialize the covariance matrix $\bSigma_0 = \lambda \Ib$ for regularization. Empirically, we found that when $\lambda$ is set to 0.001, this approach shows no substantial impact on the performance compared to the warm-up strategy.

\noindent\textbf{MaxPairUCB.}
In this algorithm, we use the same MLE estimator as that of MaxInP. However, we use a different arm selection scheme defined as follows:
\[
(\xb_t, \yb_t)=\argmax_{(\xb, \yb)\in\cA_t\times\cA_t}\Big[\langle\hat\btheta_t, \xb_t+\yb_t\rangle+\beta\|\xb_t-\yb_t\|_{\hat\bSigma_t^{-1}}\Big],
\]
which is a variant of MaxInP without the need for an arm elimination phase proposed by \citet{di2023variance}, where $\cA_t$ is the decision set at round $t$.  

\noindent\textbf{CoLSTIM.}
This technique is proposed in \citet{bengs2022stochastic}. Initially, they augment each arm with randomly perturbed utilities and select the arm with the best estimation. They argue that this procedure yields improved empirical performance. The second arm is then selected by maximizing the sum of the estimated reward and the uncertainty term proposed by \citet{saha2021optimal}.

\noindent\textbf{VACDB.}
This approach, proposed by \citet{di2023variance}, adopts a layered arm elimination process. Arms in higher layers have lower estimation uncertainty, and the elimination process starts when all arms within a layer have lower uncertainty. The arm selection scheme is similar to that of MaxPairUCB, where arms are selected to maximize the same weighted sum within a certain layer.

\noindent\textbf{FGTS.CDB.}
The proposed sampling-based algorithm for contextual dueling bandits in this paper as presented in Algorithm \ref{algorithm:FGTS.CDB}. $\eta$ and $\mu$ is set to 1 and $T^{-1 / 2} \cdot \alpha$, respectively. Empirically, we generate the model parameter $\{\btheta_t^j\}_{j = 1, 2}$ by repeatedly taking the following stochastic gradient Langevin dynamics (SGLD) step:
\begin{align} 
\btheta_t^j &\gets \btheta_t^j - \delta \nabla_{\btheta} \bigg[\underbrace{\sum_{\tau = 1}^{t - 1} L^j (\btheta_t^j, x_\tau, a_\tau^1, a_\tau^2, y_\tau) 
 - \ln p_0(\btheta_t^j)}_{I(\btheta_t^j)}\bigg]+ \sqrt{2 \delta} \bxi_i,  \label{eq:SGLD}
\end{align}
where $\bxi_i \sim \cN(0, \Ib_d)$. If we set the step size $\delta$ to a sufficiently small number, the dynamic of $\btheta_t^j$ can be regarded as the solution of the following stochastic differential equation: 
\begin{align*} 
\rd \btheta_s = - \nabla I(\btheta_s) \rd s + \sqrt{2} \rd \Wb_s, 
\end{align*} where $W_s$ is the Brownian motion in $\RR^d$ with $W_0=0$. We denote by $q_s(\btheta) \rd \btheta$ the distribution of $\btheta_s$ at time $s$. Then it is known that $q_s$ satisfies the following Fokker-Planck equation: 
\begin{align*} 
\frac{\partial q_s}{\partial s} &= \nabla \cdot \left(q_s \nabla I(\btheta_s)\right) + \lambda \Delta q_s. 
\end{align*}
The stationary distribution of $\btheta_t^j$ after SGLD iterations can then be derived by setting $\partial q_s/\partial s = 0$, from which we have $\btheta_t^j$ converges to a sample from $p^j(\btheta) \propto \exp(-I(\btheta))$. 

In our experiment, we implement \eqref{eq:SGLD} with initial step size $\delta = 0.005$. After each round, we schedule the step size $\delta$ with decaying rate $0.99$ to stabilize the optimization process. 

\begin{figure*}[ht!]
    \begin{center}
    \subfigure[$d=5$]{
    \includegraphics*[width=0.32\textwidth]{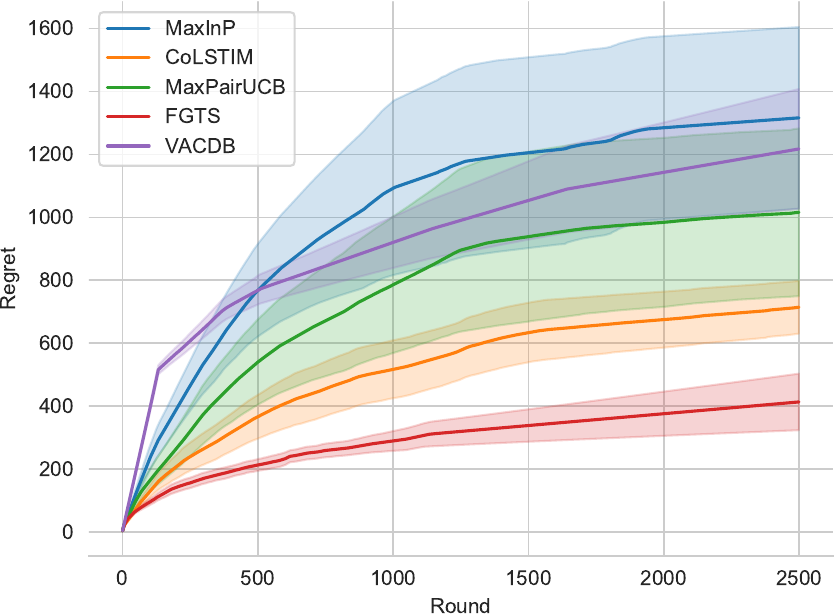}}
    \subfigure[$d=10$]{
    \includegraphics*[width=0.32\textwidth]{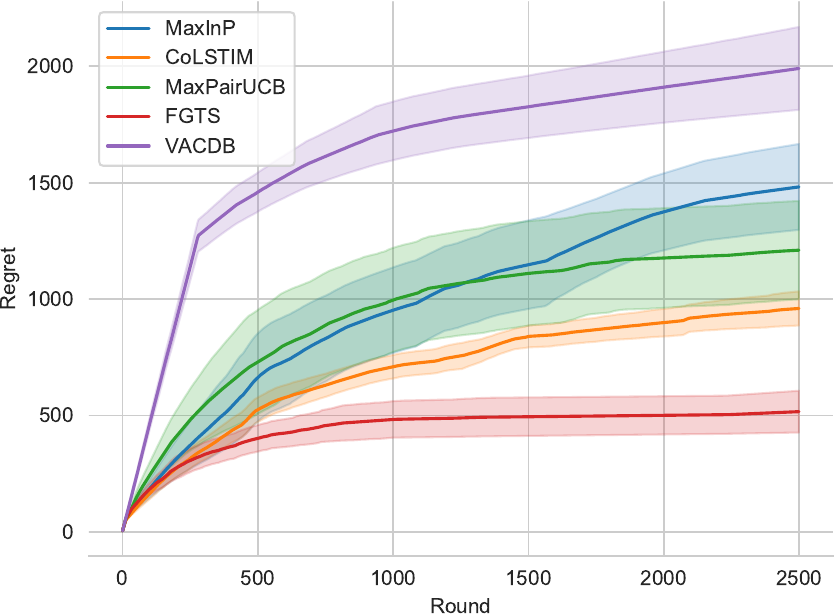}}
    \subfigure[$d=15$]{
    \includegraphics*[width=0.32\textwidth]{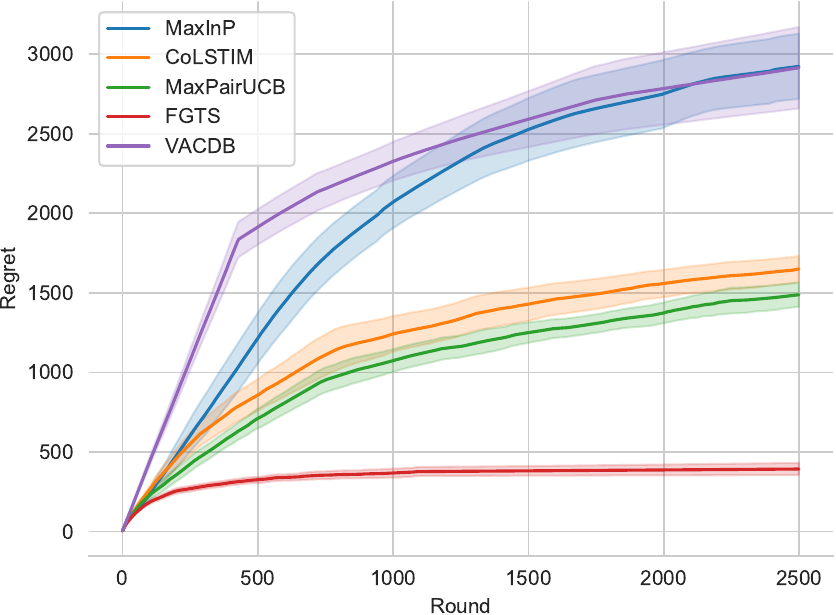}}
    \caption{Regret comparison with MaxInP, MaxPairUCB and CoLSTIM. \label{figure:1}}
    \end{center}
\end{figure*}

\subsection{Regret Comparison}
We plot the regret with respect to the number of rounds in Figure \ref{figure:1}. The results are averaged over 30 trials. In Figure \ref{figure:1}, we run FGTS.CDB with $\alpha = 0.1$. For the benchmarks, we select the hyperparameters, including the conficence radius in MaxInP and MaxPairUCB and the magnitude of perturbations in CoLSTIM, to be the best-performing hyperparameter within $\{10^{-2}, 10^{-1}, 10^0, 10^1\}$. It is shown that, FGTS.CDB outperforms the previous algorithms by a large margin with dimension $d$ set to 5, 10, 15. Additionally, the standard deviation of FGTS.CDB among different trials is also the smallest according to our experiments, which indicates that FGTS.CDB is more stable under different random seeds.

\subsection{Ablation Study}
One benefit of FGTS.CDB in empirical tasks is that the hyperparameters $\mu = T^{-1 / 2} \cdot \alpha$ and $\eta$ are independent of the feature dimension $d$, simplifying the tuning of parameters. In contrast, the confidence radius in UCB-based algorithms usually need to be carefully tuned to achieve the desirable performance~\citep{lattimore2020bandit}. In Figure \ref{figure:2}, we study the performance of FGTS.CDB when $d = 5, 10, 15$ under different values of $\alpha$. It is observed that the regret of FGTS.CDB is robust to different values of $\alpha$. Surprisingly, it turns out that FGTS.CDB performs well even without Feel-Good exploration ($\alpha = 0$). We conjecture that a Thompson-sampling-based algorithm may also work for contextual dueling bandit setting due to its stochastic nature, which also encourages the agent to explore different pairs of arms.  We leave the study of Thompson-sampling without Feel-Good exploration under contextual dueling bandit setting for future work. 

\begin{figure}[h!]
    \centering
    \subfigure[$d=5$]{
    \includegraphics*[width=0.32\textwidth]{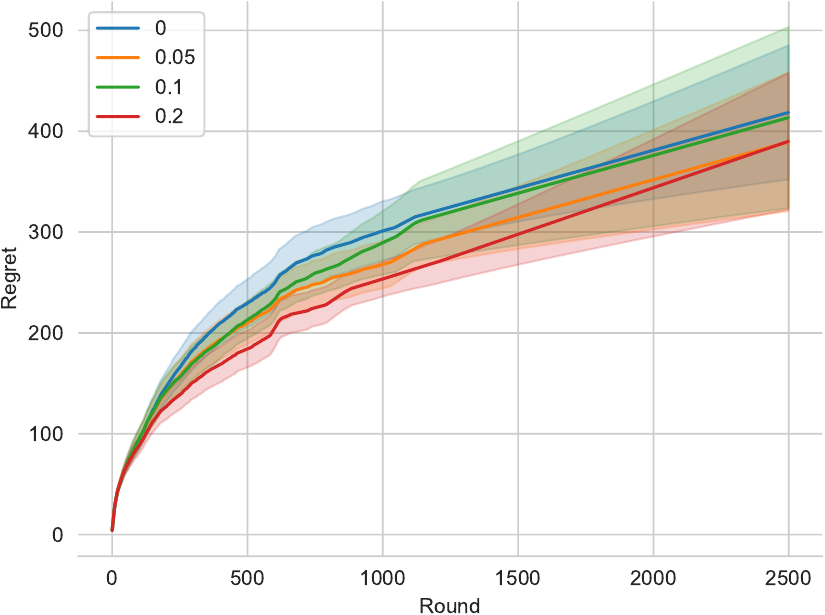}}
    \subfigure[$d=10$]{
    \includegraphics*[width=0.32\textwidth]{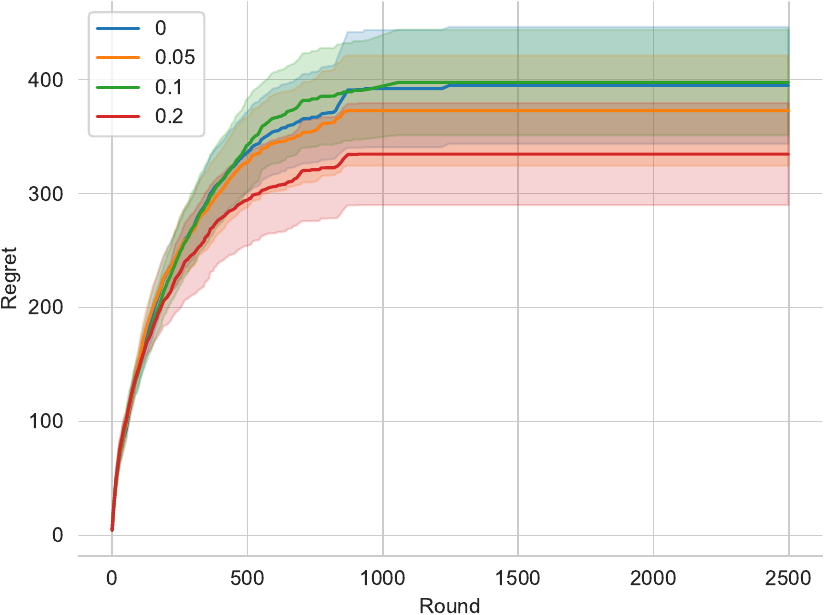}}
    \subfigure[$d=15$]{
    \includegraphics*[width=0.32\textwidth]{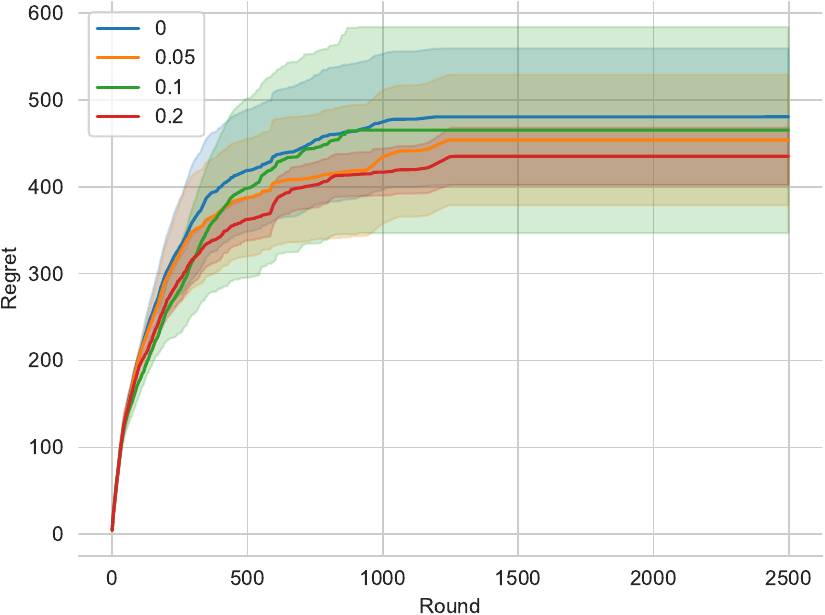}}
    \caption{Performance of FGTS.CDB over different $\alpha$.}
    \label{figure:2}
\end{figure}

\section{Conclusion}

In this work, we apply the technique of Feel-Good Thompson sampling to the setting of contextual dueling bandits.
We propose an algorithm, FGTS.CDB, whose pivotal design is a Feel-Good exploration term that contains an additional term representing the reward of the adversarial arm, compared with standard Thompson sampling.
We show that our algorithm achieves a nearly minimax-optimal regret bound.
Furthermore, experiments on synthetic data show that the performance of our algorithm based on FGTS is comparable with UCB-based algorithms.
As a future direction, it is interesting to explore the possibility of variance-aware algorithms based on the FGTS technique.
The extension of our algorithm to the setting of preference-based reinforcement learning is also an interesting topic to study.

\appendix

\section{Proof of Main Results}\label{section:proof_linear&nonlinear}

\subsection{Proof of Theorem \ref{theorem:main}}\label{section:proof_linear}

We first restate Theorem \ref{theorem:main} more precisely:
\begin{theorem}[Restatement of Theorem \ref{theorem:main}]\label{theorem:main_restate}
Assume that the hyperparameters are selected as $\mu=0.25$ and $\mu=1/(10e^B\sqrt T)$, and the logarithm of the prior distribution $p_0$ is $L$-Lipschitz. Then the expected regret is bounded by
\[
\sum_{t=1}^T\EE\bigg[r_*(x_t, a_t^*)-\frac{r_*(x_t, a_t^1)+r_*(x_t, a_t^2)}2\bigg]\le10e^Bd\sqrt T\bigg[2-\frac{\log p_0(\btheta_*)}{d}+\log\bigg(\frac{L}{\sqrt d}+\frac{1}{5e^B}\sqrt{\frac Td}+\frac{T}{2\sqrt d}\bigg)\bigg].
\]
\end{theorem}
The following Lemma bounds the regret with terms related with $Z_T^j$:
\begin{lemma}\label{lemma:main}
Under Assumption \ref{assumption:linear}, if $\eta\le1/2$, then we have
\[
\sum_{t=1}^T\EE\bigg[r_*(x_t, a_t^*)-\frac{r_*(x_t, a_t^1)+r_*(x_t, a_t^2)}2\bigg]\le\bigg(\frac{9de^{2B}}{2\eta}+32B^2\bigg)\mu T-\frac{Z_T^1+Z_T^2}{2\mu}.
\]
\end{lemma}
We need the following lemma to characterize $Z_T^j$ in Lemma \ref{lemma:main}:
\begin{lemma}[Restatement of Lemma \ref{lemma:ZT_bound_main}]\label{lemma:ZT_bound}
Assume that the log of the prior distribution $p_0$ is $L$-Lipschitz. Then for $j=1, 2$, we have
\[
Z_T^j\ge\log p_0(\btheta_*)-d-d\log\frac{L+2(\mu+\eta)T}{\sqrt d}.
\]
\end{lemma}
The proofs of Lemma \ref{lemma:main} and Lemma \ref{lemma:ZT_bound} are presented in Section \ref{subsection:main} and Section \ref{subsection:ZT_bound}, respectively. We now present the proof of Theorem \ref{theorem:main}:
\begin{proof}[Proof of Theorem \ref{theorem:main}]
Combining Lemma \ref{lemma:main} and Lemma \ref{lemma:ZT_bound}, we have
\begin{align*}
&\sum_{t=1}^T\EE\bigg[r_*(x_t, a_t^*)-\frac{r_*(x_t, a_t^1)+r_*(x_t, a_t^2)}2\bigg]\\
&\le\frac{9de^{2B}}{2\eta}\bigg(1+\frac{64B^2\eta}{de^{2B}}\bigg)\mu T+\frac1\mu\bigg[d-\log p_0(\btheta_*)+d\log\frac{L+2(\mu+\eta)T}{\sqrt d}\bigg]\\
&\le\frac{9de^{2B}}{2\eta}\bigg(1+\frac{32}{e^2}\bigg)\mu T+\frac1\mu\bigg[d-\log p_0(\btheta_*)+d\log\frac{L+2(\mu+\eta)T}{\sqrt d}\bigg]\\
&\le\frac{25e^{2B}dT\mu}{\eta}+\frac1\mu\bigg[d-\log p_0(\btheta_*)+d\log\frac{L+2(\mu+\eta)T}{\sqrt d}\bigg]\\
&=10e^Bd\sqrt T\bigg[2-\frac{\log p_0(\btheta_*)}{d}+\log\bigg(\frac{L}{\sqrt d}+\frac{1}{5e^B}\sqrt{\frac Td}+\frac{T}{2\sqrt d}\bigg)\bigg],
\end{align*}
where the first inequality holds due to Lemma \ref{lemma:main} and Lemma \ref{lemma:ZT_bound}, the second inequality holds because $B/e^B\le1/e$, the third inequality holds because $9/2\cdot(1+32/e^2)\le25$, and the equality holds by substituting $\eta=0.25$ and $\mu=1/(10e^B\sqrt T)$.
\end{proof}

\subsection{Proof of Theorem \ref{theorem:lower_bound}}

In this section, we provide the proof for Theorem~\ref{theorem:lower_bound}.
Instead of applying using the proof that is similar to the case of bounded $\ell_\infty$ norm, which is the approach used in the proof of Theorem 3.1 of~\citep{bengs2022stochastic}, we follow the proof for the case of bounded $\ell_2$-norm in the standard contextual bandit setting~\citep{lattimore2020bandit}.

\noindent\textbf{Notations.} For $\btheta\in\Theta$, let $\PP_{\btheta}$ and $\EE_{\btheta}$ be the distribution and expectation over the trajectory generated by the algorithm, respectively. Let $\phi_{ti}^j$ be the shorthand notation for $(\bphi(x_t, a_t^j))_i$. For two probability distributions $\PP_1$ and $\PP_2$, let $\DKL(\PP_1||\PP_2)$ be their KL-divergence.

\begin{proof}[Proof of Theorem~\ref{theorem:lower_bound}]
    We fix $i\in[d]$, and define
    \begin{align*}
    \tau_i\coloneqq T\wedge\min\bigg\{\tau:\sum_{t=1}^\tau[(\phi_{ti}^1)^2+(\phi_{ti}^2)^2]\ge\frac{2T}{d}\bigg\}
    \end{align*}
    For $x\in\{\pm1\}$, we define
    \begin{align*}
    U_{\btheta, i}(x)=\EE_{\btheta}\bigg[\sum_{t=1}^{\tau_i}\bigg(\frac1{\sqrt d}-\phi_{ti}^1x\bigg)^2+\sum_{t=1}^{\tau_i}\bigg(\frac1{\sqrt d}-\phi_{ti}^2x\bigg)^2\bigg],
    \end{align*}
    We fix $\btheta\in\Theta=\{\pm\Delta\}^d$ for some $\Delta$ to be determined. For the fixed $i$, define $\btheta'$ to be the vector such that $\theta'_i=-\theta_i$ and $\theta'_j=\theta_j$ for all $j\neq i$. We then have
    \begin{align*}
    U_{\btheta, i}(\sign(\theta_i))+U_{\btheta', i}(\sign(\theta'_i))&=\underbrace{U_{\btheta, i}(\sign(\theta_i))-U_{\btheta', i}(\sign(\theta_i))}_{I_1}+\underbrace{U_{\btheta', i}(\sign(\theta_i))+U_{\btheta', i}(\sign(\theta'_i))}_{I_2}.
    \end{align*}
    For $I_1$, note that
    \begin{align*}
    &\sum_{t=1}^{\tau_i}\bigg(\frac1{\sqrt d}-\phi_{ti}^1\sign(\theta_i)\bigg)^2+\sum_{t=1}^{\tau_i}\bigg(\frac1{\sqrt d}-\phi_{ti}^2\sign(\theta_i)\bigg)^2\\
    &=\frac{2\tau_i}{d}-\frac{2\sign(\theta_i)}{\sqrt d}\sum_{t=1}^{\tau_i}(\phi_{ti}^1+\phi_{ti}^2)+\sum_{t=1}^{\tau_i}[(\phi_{ti}^1)^2+(\phi_{ti}^2)^2]\\
    &\le\frac{4\tau_i}{d}+2\sum_{t=1}^{\tau_i}[(\phi_{ti}^1)^2+(\phi_{ti}^2)^2]\\
    &\le\frac{4T}d+2\cdot\frac{2T}{d}+2\cdot(1+1)=\frac{8T}{d}+4,
    \end{align*}
    where first inequality holds due to Cauchy-Schwarz inequality, and the second inequality holds due to the definition of $\tau_i$ and $|\phi_{ti}^j|\le1$. We thus bound $I_1$ as follows:
    \begin{align}
    I_1&=\EE_{\btheta}\bigg[\sum_{t=1}^{\tau_i}\bigg(\frac1{\sqrt d}-\phi_{ti}^1\sign(\theta_i)\bigg)^2+\sum_{t=1}^{\tau_i}\bigg(\frac1{\sqrt d}-\phi_{ti}^2\sign(\theta_i)\bigg)^2\bigg]\nonumber\\
    &\qquad-\EE_{\btheta'}\bigg[\sum_{t=1}^{\tau_i}\bigg(\frac1{\sqrt d}-\phi_{ti}^1\sign(\theta_i)\bigg)^2+\sum_{t=1}^{\tau_i}\bigg(\frac1{\sqrt d}-\phi_{ti}^2\sign(\theta_i)\bigg)^2\bigg]\nonumber\\
    &\ge-4(2T/d+1)\sqrt{\DKL(\Ber(\PP_{\theta}||\PP_{\theta'}))/2}\nonumber\\
    &\ge-(2T/d+1)\sqrt{\EE_{\btheta}\bigg[\sum_{t=1}^{\tau_i}\langle\btheta-\btheta', \bphi(x_t, a_t^1)-\bphi(x_t, a_t^2)\rangle^2\bigg]}\nonumber\\
    &\ge-2\Delta(2T/d+1)\sqrt{2\EE_{\btheta}\bigg[\sum_{t=1}^{\tau_i}((\phi_{ti}^1)^2+(\phi_{ti}^2)^2)\bigg]}\nonumber\\
    &\ge-4\Delta(2T/d+1)\sqrt{T/d+1}\nonumber\\
    &\ge-12\sqrt2\Delta(T/d)^{1.5},\label{eq:bound_I1}
    \end{align}
    where the first ineequality holds due to Pinker's inequality, the second inequality holds due to Lemma \ref{lemma:DKL}, the third inequality holds due to the definition of $\btheta'$ and because $|\phi_{ti}^1-\phi_{ti}^2|^2\le2(\phi_{ti}^1)^2+2(\phi_{ti}^2)^2$, the fourth inequality holds due to the definition of $\tau_i$ and $|\phi_{ti}^j|\le1$, and the last inequality holds because $T\ge d$.
    For $I_2$, we have
    \begin{align}
    I_2&=\EE_{\btheta'}\bigg[\sum_{t=1}^{\tau_i}\bigg(\frac1{\sqrt d}-\phi_{ti}^1\bigg)^2+\sum_{t=1}^{\tau_i}\bigg(\frac1{\sqrt d}-\phi_{ti}^2\bigg)^2\bigg]+\EE_{\theta'}\bigg[\sum_{t=1}^{\tau_i}\bigg(\frac1{\sqrt d}+\phi_{ti}^1\bigg)^2+\sum_{t=1}^{\tau_i}\bigg(\frac1{\sqrt d}+\phi_{ti}^2\bigg)^2\bigg]\nonumber\\
    &=\frac{4\tau_i}{d}+2\EE_{\btheta'}\bigg[\sum_{t=1}^{\tau_i}[(\phi_{ti}^1)^2+(\phi_{ti}^2)^2]\bigg]\ge\frac{4T}{d},\label{eq:bound_I2}
    \end{align}
    where the inequality holds due to the definition of $\tau_i$. Combining \eqref{eq:bound_I1} and \eqref{eq:bound_I2}, we have
    \begin{align*}
    U_{\btheta, i}(\sign(\theta_i))+U_{\btheta', i}(\sign(\theta_i'))\ge\frac{4T}{d}-12\sqrt2\Delta(T/d)^{1.5}.
    \end{align*}
    Thus, by pairing each $\btheta$ with the corresponding $\btheta'$ for each $i$, we have
    \begin{align*}
    \frac1{|\Theta|}\sum_{\btheta\in\Theta}\sum_{i=1}^dU_{\btheta, i}(\sign(\theta_i))\ge2T-6\sqrt2\Delta T^{1.5}d^{-0.5}.
    \end{align*}
    Therefore, there exists $\btheta\in\Theta$ such that
    \begin{align}
    \sum_{i=1}^dU_{\btheta, i}(\sign(\theta_i))\ge2T-6\sqrt2\Delta T^{1.5}d^{-0.5}.\label{eq:sum_U}
    \end{align}
    For this $\btheta$, the regret is
    \begin{align}
    \regret(T)&=\frac\Delta2\cdot\EE_{\btheta}\bigg[\sum_{t=1}^T\sum_{i=1}^d\bigg(\frac1{\sqrt d}-\phi_{ti}^1\sign(\theta_i)\bigg)\bigg]+\frac\Delta2\cdot\EE_{\btheta}\bigg[\sum_{t=1}^T\sum_{i=1}^d\bigg(\frac1{\sqrt d}-\phi_{ti}^2\sign(\theta_i)\bigg)\bigg].\label{eq:regret_theta}
    \end{align}
    Note that
    \begin{align}
    \sum_{i=1}^d\bigg(\frac1{\sqrt d}-\phi_{ti}^j\sign(\theta_i)\bigg)&=\sum_{i=1}^d\bigg(\frac1{2\sqrt d}-\phi_{ti}^j\sign(\theta_i)\bigg)+\frac{\sqrt d}2\nonumber\\
    &\ge\sum_{i=1}^d\bigg(\frac1{2\sqrt d}-\phi_{ti}^j\sign(\theta_i)\bigg)+\frac{\sqrt d}2\sum_{i=1}^d(\phi_{ti}^j)^2\nonumber\\
    &=\frac{\sqrt d}{2}\sum_{i=1}^d\bigg(\frac1{\sqrt d}-\phi_{ti}^j\sign(\theta_i)\bigg)^2,\label{eq:to_square}
    \end{align}
    where the inequality holds because $\sum_{i=1}^d(\phi_{ti}^j)^2\le1$. Plugging \eqref{eq:to_square} into \eqref{eq:regret_theta}, we have
    \begin{align*}
    \regret(T)&\ge\frac{\Delta\sqrt d}4\sum_{i=1}^d\EE_{\btheta}\bigg[\sum_{t=1}^T\bigg(\frac1{\sqrt d}-\phi_{ti}^1\sign(\theta_i)\bigg)^2+\sum_{t=1}^T\bigg(\frac1{\sqrt d}-\phi_{ti}^1\sign(\theta_i)\bigg)^2\bigg]\\
    &\ge\frac{\Delta\sqrt d}4\sum_{i=1}^d\EE_{\btheta}\bigg[\sum_{t=1}^{\tau_i}\bigg(\frac1{\sqrt d}-\phi_{ti}^1\sign(\theta_i)\bigg)^2+\sum_{t=1}^{\tau_i}\bigg(\frac1{\sqrt d}-\phi_{ti}^1\sign(\theta_i)\bigg)^2\bigg]\\
    &=\frac{\Delta\sqrt d}4\sum_{i=1}^dU_{\btheta, i}(\sign(\theta_i))\\
    &\ge\frac{\Delta\sqrt d}{2}(T-3\sqrt 2\Delta T^{1.5}d^{-0.5}),
    \end{align*}
    where the second inequality holds because $\tau_i\le T$, and the last inequality holds due to \eqref{eq:sum_U}. Let $\Delta=\frac16\sqrt{\frac{d}{2T}}$, then
    \begin{align*}
    \regret(T)\ge\frac{d\sqrt T}{24\sqrt2}.
    \end{align*}
\end{proof}

\subsection{Proof of Theorem \ref{theorem:nonlinear}}\label{section:proof_nonlinear}

Similar to the proof of Theorem \ref{theorem:main}, we make the following notations:
\begin{gather*}
\BE_t^j\coloneqq\Delta r_{\theta_t^j}(x_t, a_t^j, a_t^{3-j})-\Delta r_{\theta_*}(x_t, a_t^j, a_t^{3-j}),\\
\FG_t^j(\theta)=\max_{a\in\cA}\Delta r_{\theta}(x_t, a, a_t^{3-j})-\Delta r_{\theta_*}(x_t, a_t^*, a_t^{3-j}),\\
\LS_t(\theta)=(\Delta r_\theta(x_t, a_t^j, a_t^{3-j})-\Delta r_{\theta_*}(x_t, a_t^j, a_t^{3-j}))^2.
\end{gather*}
We first present the following restatement of Theorem \ref{theorem:nonlinear}:
\begin{theorem}[Restatement of Theorem \ref{theorem:nonlinear}]\label{theorem:nonlinear_restate}
Assume that the hyperparameters are selected as $\eta=0.25$, and
\[
\mu=\frac{e^{-B}}{5}\sqrt{\frac{\log1/[p_0(\theta_*)\bar\mu(\{\theta\in\Theta:d(\theta, \theta_*)\le2/(L_0T)\})]}{T\dc}},
\]
We also assume that the $\log p_0$ is $L$-Lipschitz and that $r_\theta$ is $L_0$-Lipschitz in $\theta$. Then the regret is bound by
\begin{align*}
\EE[\regret(T)]&\le4+5e^B\sqrt{T\dc\log1/[p_0(\theta_*)\bar\mu(\{\theta\in\Theta:d(\theta, \theta_*)\le2/(L_0T)\})]}\\
&\qquad\cdot\bigg[2+\frac{1+2L/(L_0T)}{\log1/[p_0(\theta_*)\bar\mu(\{\theta\in\Theta:d(\theta, \theta_*)\le2/(L_0T)\})]}\bigg].
\end{align*}
\end{theorem}
Note that Lemma \ref{lemma:bound_LS-FG} does not require linearity of the reward function and can be directly applied in the proof. We only require the following lemma as a counterpart of Lemma \ref{lemma:ZT_bound}:
\begin{lemma}\label{lemma:ZT_bound_nonlinear}
Assume that $\log p_0$ is $L$-Lipschitz and that $r_\theta$ is $L_0$-Lipschitz in $\theta$. Then for $j=1, 2$, we have
\[
Z_T^j\ge\log p_0(\theta_*)+\log\bar\mu(\{\theta\in\Theta:d(\theta, \theta_*)\le2/(L_0T)\})-\frac{2L}{L_0T}-4(\eta+\mu).
\]
\end{lemma}
The proof of Lemma \ref{lemma:ZT_bound_nonlinear} is given in Appendix \ref{subsection:ZT_bound_nonlinear}. We now present the proof of Theorem \ref{theorem:nonlinear_restate}:
\begin{proof}[Proof of Theorem \ref{theorem:nonlinear_restate}]
$\BE_t^j$ can be bounded as
\begin{align}
\EE[\BE_t^j-\FG_t^j(\theta_t^j)]&\le\frac{9e^{2B}\mu\dc}{2\eta}+\EE_{S_{t-1}, x_t, a_t^1, a_t^2}\EE_{\tilde\theta\sim p^j(\cdot|S_{t-1})}\bigg[\frac{e^{-2B}\eta}{18\mu}\LS_t(\tilde\theta)-\FG_t^j(\tilde\theta)\bigg]\nonumber\\
&\le\frac{9e^{2B}\mu\dc}{2\eta}+\frac{Z_{t-1}^j-Z_t^j}{\mu}+32\mu B^2,\nonumber
\end{align}
where the first inequality holds due to the definition of $\dc$, and the second inequality holds due to Lemma \ref{lemma:bound_LS-FG}. Taking the sum over $t$ and substituting $\eta=0.25$, we have
\begin{align*}
&\sum_{t=1}^T\EE[\BE_t^j-\FG_t^j(\theta_t^j)]\le(18e^{2B}\dc+32B^2)\mu T+\frac{e^{-2B}\epsilon}{72\mu}-\frac{Z_T^j}{\mu}\le(18+32/e^2)e^{2B}\dc\mu T-\frac{Z_T^j}\mu\\
&\le25e^{2B}\dc\mu T+\frac1\mu+\frac{2L}{L_0T\mu}+4-\frac{\log p_0(\theta_*)}\mu+\frac{\log1/\bar\mu(\{\theta\in\Theta:d(\theta, \theta_*)\le2/(L_0T)\})}{\mu},
\end{align*}
where the second inequality holds because $e^{-B}B\le1/e$, the second inequality holds because $18+32/e^2\le25$, and the last inequality holds due to Lemma \ref{lemma:ZT_bound_nonlinear}. Taking
\begin{align*}
\mu=\frac{e^{-B}}{5}\sqrt{\frac{\log1/[p_0(\theta_*)\bar\mu(\{\theta\in\Theta:d(\theta, \theta_*)\le2/(L_0T)\})]}{T\dc}},
\end{align*}
then we have
\begin{align*}
\sum_{t=1}^T\EE[\BE_t^j-\FG_t^j(\theta_t^j)]&\le4+5e^B\sqrt{T\dc\log1/[p_0(\theta_*)\bar\mu(\{\theta\in\Theta:d(\theta, \theta_*)\le2/(L_0T)\})]}\\
&\qquad\cdot\bigg[2+\frac{1+2L/(L_0T)}{\log1/[p_0(\theta_*)\bar\mu(\{\theta\in\Theta:d(\theta, \theta_*)\le2/(L_0T)\})]}\bigg].
\end{align*}
\end{proof}

\section{Proof of Lemmas in Appendix \ref{section:proof_linear&nonlinear}}\label{section:proof_A&B}

\subsection{Proof of Lemma \ref{lemma:main}}\label{subsection:main}

In order to prove Lemma \ref{lemma:main}, we first present the following lemma, which connects $\LS_t$ and $\FG_t^j$ with the difference of the potential $Z_t^j$:
\begin{lemma}[Restatement of Lemma \ref{lemma:LS-FG_main}]\label{lemma:bound_LS-FG}
Under the same assumptions as Lemma \ref{lemma:main}, we have
\[
\EE_{S_{t-1}, x_t, a_t^1, a_t^2}\EE_{\tilde\btheta\sim p^j(\cdot|S_{t-1})}\bigg[\frac{e^{-2B}\eta}{18\mu}\LS_t(\tilde\btheta)-\FG_t^j(\tilde\btheta)\bigg]\le\mu^{-1}(Z_{t-1}^j-Z_t^j)+32\mu B^2.
\]
\end{lemma}

The proof of Lemma \ref{lemma:bound_LS-FG} is given in Appendix \ref{subsection:bound_LS-FG}. We now present the proof of Lemma \ref{lemma:main}:
\begin{proof}[Proof of Lemma \ref{lemma:main}]
The regret at step $t$ can be decomposed as
\begin{align}
&r_*(x_t, a_t^*)-\frac{r_*(x_t, a_t^1)+r_*(x_t, a_t^2)}2=\EE\Big\langle\btheta_*, \bphi(x_t, a_t^*)-\frac{\bphi(x_t, a_t^1)+\bphi(x_t, a_t^2)}2\Big\rangle\nonumber\\
&=\frac12\underbrace{\big[\langle\btheta_t^1-\btheta_*, \bphi(x_t, a_t^1)-\bphi(x_t, a_t^2)\rangle-\langle\btheta_t^1, \bphi(x_t, a_t^1)-\bphi(x_t, a_t^2)\rangle+\langle\btheta_*, \bphi(x_t, a_t^*)-\bphi(x_t, a_t^2)\rangle\big]}_{I_1}\nonumber\\
&\qquad+\frac12\underbrace{\big[\langle\btheta_t^2-\btheta_*, \bphi(x_t, a_t^2)-\bphi(x_t, a_t^1)\rangle-\langle\btheta_t^2, \bphi(x_t, a_t^2)-\bphi(x_t, a_t^1)\rangle+\langle\btheta_*, \bphi(x_t, a_t^*)-\bphi(x_t, a_t^1)\rangle\big]}_{I_2}.\label{eq:decompose_regret}
\end{align}
The expectation of $I_1$ can be bounded as
\begin{align}
&\EE\big[\langle\btheta_t^1-\btheta_*, \bphi(x_t, a_t^1)-\bphi(x_t, a_t^2)\rangle-\langle\btheta_t^1, \bphi(x_t, a_t^1)-\bphi(x_t, a_t^2)\rangle+\langle\btheta_*, \bphi(x_t, a_t^*)-\bphi(x_t, a_t^2)\rangle\big]\nonumber\\
&=\EE_{S_{t-1}, x_t}\EE_{\btheta_t^1, a_t^1, a_t^2|S_{t-1}}\langle\btheta_t^1-\btheta_*, \bphi(x_t, a_t^1)-\bphi(x_t, a_t^2)\rangle-\EE_{S_{t-1}, x_t}\EE_{\btheta_t^1, a_t^2|S_{t-1}}\FG_t^1(\btheta_t^1)\nonumber\\
&\le\frac{9e^{2B}\mu d}{2\eta}+\frac{e^{-2B}\eta}{18\mu}\EE_{S_{t-1}, x_t}\EE_{a_t^1, a_t^2|S_{t-1}}\EE_{\tilde\btheta\sim p^1(\cdot|S_{t-1})}\LS_t(\tilde\btheta)-\EE_{S_{t-1}, x_t}\EE_{\btheta_t^1, a_t^2|S_{t-1}}\FG_t^1(\btheta_t^1)\nonumber\\
&=\frac{9e^{2B}\mu d}{2\eta}+\EE_{S_{t-1}, x_t, a_t^1, a_t^2}\EE_{\tilde\btheta\sim p^1(\cdot|S_{t-1})}\Big[\frac{e^{-2B}\eta}{18\mu}\LS_t(\tilde\btheta)-\FG_t^1(\tilde\btheta)\Big]\nonumber\\
&\le\bigg(\frac{9de^{2B}}{2\eta}+32B^2\bigg)\mu+\frac{Z_{t-1}^1-Z_t^1}{\mu},\label{eq:arm1_bound}
\end{align}
where the first equality holds due to the law of total expectation, the first inequality holds due to the decoupling lemma (Lemma \ref{lemma:decoupling}), the second equality holds because $\btheta_t^0$ and $a_t^1$ are independent conditioned on $S_{t-1}, x_t$, and the last inequality holds due to Lemma \ref{lemma:bound_LS-FG}. For $I_2$, we can similarly prove that
\begin{align}
&\EE\big[\langle\btheta_t^2-\btheta_*, \bphi(x_t, a_t^2)-\bphi(x_t, a_t^1)\rangle-\langle\btheta_t^2, \bphi(x_t, a_t^2)-\bphi(x_t, a_t^1)\rangle+\langle\btheta_*, \bphi(x_t, a_t^*)-\bphi(x_t, a_t^1)\rangle\big]\nonumber\\
&\le\bigg(\frac{9de^{2B}}{2\eta}+32B^2\bigg)\mu+\frac{Z_{t-1}^2-Z_t^2}{\mu}\label{eq:arm2_bound}
\end{align}
Plugging \eqref{eq:arm1_bound} and \eqref{eq:arm2_bound} into \eqref{eq:decompose_regret}, taking the sum over $t$, we have
\begin{align*}
\sum_{t=1}^T\EE\bigg[r_*(x_t, a_t^*)-\frac{r_*(x_t, a_t^1)+r_*(x_t, a_t^2)}2\bigg]&\le\bigg(\frac{9de^{2B}}{2\eta}+32B^2\bigg)\mu T+\frac{(Z_0^1+Z_0^2)-(Z_T^1+Z_T^2)}{2\mu}\\
&=\bigg(\frac{9de^{2B}}{2\eta}+32B^2\bigg)\mu T-\frac{Z_T^1+Z_T^2}{2\mu},
\end{align*}
where the equality holds because $Z_0^j=0$.
\end{proof}

\subsection{Proof of Lemma \ref{lemma:ZT_bound}}\label{subsection:ZT_bound}

\begin{proof}[Proof of Lemma \ref{lemma:ZT_bound}]
We first can decompose $\Delta L^j$ into its expectation $I_1$ and a deviation term $I_2$:
\begin{align}
\Delta L^j(\btheta, x_t, a_t^1, a_t^2, y_t)&=\underbrace{\EE_{y_t|x_t, a_t^1, a_t^2}\Delta L^j(\btheta, x_t, a_t^1, a_t^2, y_t)}_{I_1}\nonumber\\
&\qquad+\underbrace{\Delta L^j(\btheta, x_t, a_t^1, a_t^2, y_t)-\EE_{y_t|x_t, a_t^1, a_t^2}\Delta L^j(\btheta, x_t, a_t^1, a_t^2, y_t)}_{I_2}.\label{eq:decompose_DeltaL}
\end{align}
For $I_1$, note that
\[
\frac{1}{1+e^z}=-\sigma'(z),\quad\frac{1}{1+e^{-z}}=1-\frac{1}{1+e^z}=1+\sigma'(z),
\]
so we have
\begin{align}
&I_1-\mu\FG_t^j(\btheta)\nonumber\\
&=[1+\sigma'(\langle\btheta_*, \bphi(x_t, a_t^1)-\bphi(x_t, a_t^2)\rangle)]\nonumber\\
&\qquad\cdot\eta[\sigma(\langle\btheta, \bphi(x_t, a_t^1)-\bphi(x_t, a_t^2)\rangle)-\sigma(\langle\btheta_*, \bphi(x_t, a_t^1)-\bphi(x_t, a_t^2)\rangle)]\nonumber\\
&\qquad-\sigma'(\langle\btheta_*, \bphi(x_t, a_t^1)-\bphi(x_t, a_t^2)\rangle)\nonumber\\
&\qquad\cdot\eta[\sigma(-\langle\btheta, \bphi(x_t, a_t^1)-\bphi(x_t, a_t^2)\rangle)-\sigma(-\langle\btheta_*, \bphi(x_t, a_t^1)-\bphi(x_t, a_t^2)\rangle)]\nonumber\\
&=\eta[\sigma(\langle\btheta, \bphi(x, a^1)-\bphi(x, a^2)\rangle)-\sigma(\langle\btheta_*, \bphi(x, a^1)-\bphi(x, a^2)\rangle)]\nonumber\\
&\qquad-\eta\sigma'(\langle\btheta_*, \bphi(x_t, a_t^1)-\bphi(x_t, a_t^2)\rangle)\cdot\langle\btheta-\btheta_*, \bphi(x_t, a_t^1)-\bphi(x_t, a_t^2)\rangle,\label{eq:I1-muFG}
\end{align}
where the second equality holds because $\sigma(z)-\sigma(-z)=-z$. By Taylor expansion, there exists $\xi$ between $\langle\btheta, \bphi(x_t, a_t^1)-\bphi(x_t, a_t^2)\rangle$ and $\langle\btheta_*, \bphi(x_t, a_t^1)-\bphi(x_t, a_t^2)\rangle$ such that
\begin{align}
&\sigma(\langle\btheta, \bphi(x, a^1)-\bphi(x, a^2)\rangle)-\sigma(\langle\btheta_*, \bphi(x, a^1)-\bphi(x, a^2)\rangle)\nonumber\\
&\qquad-\sigma'(\langle\btheta_*, \bphi(x_t, a_t^1)-\bphi(x_t, a_t^2)\rangle)\cdot\langle\btheta-\btheta_*, \bphi(x_t, a_t^1)-\bphi(x_t, a_t^2)\rangle\nonumber\\
&=\frac{\sigma''(\xi)}2\langle\btheta-\btheta_*, \bphi(x_t, a_t^1)-\bphi(x_t, a_t^2)\rangle^2\nonumber\\
&\le2\sigma''(\xi)\|\btheta-\btheta_*\|_2^2\nonumber\\
&\le\frac12\|\btheta-\btheta_*\|^2,\label{eq:taylor}
\end{align}
where the first inequality holds because $\langle\btheta-\btheta_*, \bphi(x_t, a_t^1)-\bphi(x_t, a_t^2)\rangle^2\le\|\btheta-\btheta_*\|_2^2\cdot\|\bphi(x_t, a_t^1)-\bphi(x_t, a_t^2)\|_2^2\le4\|\btheta-\btheta_*\|^2$, and the second inequality holds because $\sigma''(\xi)=1/(2+e^\xi+e^{-\xi})\le1/4$. Furthermore, for $\FG_t^j(\btheta)$, let $\hat a=\argmax_{a\in\cA}\langle\btheta, \bphi(x_t, a)\rangle$, then
\begin{align}
\FG_t^j(\btheta)&=\max_{a\in\cA}\langle\btheta, \bphi(x_t, a)\rangle-\langle\btheta_*, \bphi(x_t, a_t^*)\rangle+\langle\btheta_*-\btheta, \bphi(x_t, a_t^{3-j})\rangle\nonumber\\
&\le\langle\btheta-\btheta_*, \bphi(x_t, \hat a)-\bphi(x_t, a_t^{3-j})\rangle\nonumber\\
&\le2\|\btheta-\btheta_*\|_2,\label{eq:FGtj_bound}
\end{align}
where the first inequality holds because $\langle\btheta_*, \bphi(x_t, a_t^*)\rangle\ge\langle\btheta_*, \bphi(x_t, \hat a)\rangle$, and the second inequality holds because $\langle\btheta-\btheta_*, \bphi(x_t, \hat a)-\bphi(x_t, a_t^{3-j})\rangle\le\|\btheta-\btheta_*\|_2\cdot\|\bphi(x_t, \hat a)-\bphi(x_t, a_t^{3-j})\|_2\le2\|\btheta-\btheta_*\|_2$. Plugging \eqref{eq:taylor} and \eqref{eq:FGtj_bound} into \eqref{eq:I1-muFG}, we have
\begin{align}\label{eq:exp_bound}
I_1\le\frac\eta2\|\btheta-\btheta_*\|_2^2+2\mu\|\btheta-\btheta_*\|_2.
\end{align}
Finally, for $I_2$, note that for $y\in\{\pm1\}$, we have
\[
\sigma(yp)-\sigma(yq)=\sigma(p)-\sigma(q)+\frac{y-1}{2}(p-q),
\]
so we have
\begin{align}
I_2&=\eta\bigg[\frac{y_t-1}{2}-\sigma'(\langle\btheta_*, \bphi(x_t, a_t^1)-\bphi(x_t, a_t^2)\rangle)\bigg]\langle\btheta-\btheta_*, \bphi(x_t, a_t^1)-\bphi(x_t, a_t^2)\rangle\nonumber\\
&\le\eta\bigg|\frac{y_t-1}{2}-\sigma'(\langle\btheta_*, \bphi(x_t, a_t^1)-\bphi(x_t, a_t^2)\rangle)\bigg|\cdot2\|\btheta-\btheta_*\|_2,\label{eq:dev_bound}
\end{align}
where the inequality holds because $\langle\btheta-\btheta_*, \bphi(x_t, a_t^1)-\bphi(x_t, a_t^2)\rangle\le\|\btheta-\btheta_*\|_2\cdot\|\bphi(x_t, a_t^1)-\bphi(x_t, a_t^2)\|_2\le2\|\btheta-\btheta_*\|_2$. Denote
\[
\epsilon_t\coloneqq\bigg|\frac{y_t-1}{2}-\sigma'(\langle\btheta_*, \bphi(x_t, a_t^1)-\bphi(x_t, a_t^2)\rangle)\bigg|,
\]
then we have
\begin{align}\label{eq:abs_noise_bound}
\EE_{y_t|x_t, a_t^1, a_t^2}\epsilon_t=\frac{2\exp(\langle\btheta_*, \bphi(x_t, a_t^1)-\bphi(x_t, a_t^2)\rangle)}{[1+\exp(\langle\btheta_*, \bphi(x_t, a_t^1)-\bphi(x_t, a_t^2)\rangle)]^2}\le\frac12,
\end{align}
where the inequality holds due to AM-GM inequality. Denote
\[
\hat\btheta=\min_{\btheta:\|\btheta-\btheta_*\|_2\le\delta}p_0(\btheta),
\]
then we have
\begin{align}
\log\int_{\|\tilde\btheta-\btheta_*\|_2\le\delta}p_0(\tilde\btheta)\ud\tilde\btheta&\ge\log\int_{\|\tilde\btheta-\btheta_*\|_\infty\le\delta/\sqrt d}p_0(\tilde\btheta)\ud\btheta\nonumber\\
&\ge\log\Big(p_0(\hat\btheta)(2\delta/\sqrt d)^d\Big)\nonumber\\
&\ge\log p_0(\btheta_*)-L\delta+d\log(2\delta/\sqrt d),\label{eq:central_probability}
\end{align}
where the first inequality holds because $\{\btheta:\|\btheta-\btheta_*\|_\infty\le\delta/\sqrt d\}\subset\{\btheta:\|\btheta-\btheta_*\|_2\le\delta\}$, the second inequality holds due to the definition of $\hat\btheta$, and the last inequality holds because $\log p_0$ is $L$-Lipschitz. Therefore, we have
\begin{align}
Z_T^j&=\EE_{S_T}\log\EE_{\tilde\btheta\sim p_0}\exp\bigg(-\sum_{t=1}^T\Delta L^j(\tilde\btheta, x_t, a_t^1, a_t^2, y_t)\bigg)\nonumber\\
&\ge\EE_{S_T}\log\EE_{\tilde\btheta\sim p_0}\exp\bigg(-\frac{\eta T}{2}\|\tilde\btheta-\btheta_*\|_2^2-2\mu T\|\tilde\btheta-\btheta_*\|_2-2\eta\sum_{t=1}^T\epsilon_t\cdot\|\tilde\btheta-\btheta_*\|_2\bigg)\nonumber\\
&\ge\EE_{S_T}\log\int_{\|\tilde\btheta-\btheta_*\|_2\le\delta}p_0(\tilde\btheta)\exp\bigg(-\frac{\eta T\delta^2}{2}-2\mu T\delta-2\eta\delta\sum_{t=1}^T\epsilon_t\bigg)\ud\tilde\btheta\nonumber\\
&=-\frac{\eta T\delta^2}{2}-2\mu T\delta-2\eta\delta\sum_{t=1}^T\EE\epsilon_t+\log\int_{\|\tilde\btheta-\btheta_*\|_2\le\delta}p_0(\tilde\btheta)\ud\tilde\btheta\nonumber\\
&\ge-2(\mu+\eta)T\delta+\log p_0(\btheta_*)-L\delta+d\log(2\delta/\sqrt d),\nonumber
\end{align}
where the first inequality holds by plugging \eqref{eq:exp_bound} and \eqref{eq:dev_bound} into \eqref{eq:decompose_DeltaL}, the second inequality holds by restricting the domain to $\{\tilde\btheta:\|\tilde\btheta-\btheta_*\|_2\le\delta\}$, and the last inequality holds due to \eqref{eq:abs_noise_bound} and \eqref{eq:central_probability} and because $\delta/2\le1$. Take $\delta=\min\{d/(L+2(\mu+\eta)T), B, 2\}$, then
\[
Z_T^j\ge\log p_0(\btheta_*)-d+d\log\frac{\sqrt d}{L+2(\mu+\eta)T}.
\]
\end{proof}

\subsection{Proof of Lemma \ref{lemma:ZT_bound_nonlinear}}\label{subsection:ZT_bound_nonlinear}
\begin{proof}[Proof of Lemma \ref{lemma:ZT_bound_nonlinear}]
We first can decompose $\Delta L^j$ into its expectation $I_1$ and a deviation term $I_2$:
\begin{align}\label{eq:decompose_DeltaL_nonlinear}
\Delta L^j(\theta, x_t, a_t^1, a_t^2, y_t)&=\underbrace{\EE_{y_t|x_t, a_t^1, a_t^2}\Delta L^j(\theta, x_t, a_t^1, a_t^2, y_t)}_{I_1}\\
&\qquad+\underbrace{\Delta L^j(\theta, x_t, a_t^1, a_t^2, y_t)-\EE_{y_t|x_t, a_t^1, a_t^2}\Delta L^j(\theta, x_t, a_t^1, a_t^2, y_t)}_{I_2}.
\end{align}
For $I_1$, similar to the proof of Lemma \ref{lemma:ZT_bound}, we have
\begin{align}\label{eq:exp_bound_nonlinear}
I_1-\mu\FG_t^j(\theta)&\le\frac\eta8(\Delta r_{\theta}(x_t, a_t^1, a_t^2)-\Delta r_{\theta_*}(x_t, a_t^1, a_t^2))^2\le\frac{L_0^2\eta}{2}d(\theta, \theta_*)^2,
\end{align}
where the second inequality holds because $\Delta r_\theta$ is $2L_0$-Lipschitz. For $\FG_j^t(\btheta)$, let $\hat a=\argmax_{a\in\cA}r_\theta(x_t, a)$, then
\begin{align}
\FG_t^j(\theta)=\Delta r_\theta(x_t, \hat a, a_t^{3-j})-\Delta r_{\theta_*}(x_t, a_t^*, a_t^{3-j})\le\Delta r_\theta(x_t, \hat a, a_t^{3-j})-\Delta r_{\theta_*}(x_t, \hat a, a_t^{3-j})\le 2L_0d(\theta, \theta_*),\label{eq:FG_bound_nonlinear}
\end{align}
where the first inequality holds due to optimality of $a_t^*$, and the second inequality holds because $\Delta r_\theta$ is $2L_0$-Lipschitz. For $I_2$, similar to the proof of Lemma \ref{lemma:ZT_bound}, we have
\begin{align}\label{eq:I2_nonlinear}
|I_2|\le\frac\eta2|\Delta r_\theta(x_t, a_t^1, a_t^2)-\Delta r_{\theta_*}(x_t, a_t^1, a_t^2)|\le L_0\eta d(\theta, \theta_*),
\end{align}
where the second inequality holds because $\Delta r_\theta$ is $2L_0$-Lipschitz. Denote
\[
\hat\theta=\min_{\theta:d(\theta, \theta_*)\le\delta}p_0(\theta),
\]
then we have
\begin{align}
\log\int_{d(\tilde\theta, \theta_*)\le\delta}p_0(\tilde\theta)\ud\theta\ge\log \bigg(p_0(\hat\theta)\int_{d(\tilde\theta, \theta_*)\le\delta}\ud\tilde\theta\bigg)\ge\log p_0(\theta_*)-L\delta+\log\bar\mu(\{\theta\in\Theta:d(\theta, \theta_*)\le\delta\}),\label{eq:measure_nonlinear}
\end{align}
where the first inequality holds due to the definition of $\hat\mu$, and the second inequality holds because $\log p_0$ is $L$-Lipschitz. We thus have
\begin{align*}
Z_T^j&=\EE_{S_T}\log\EE_{\tilde\theta\sim p_0}\exp\bigg(-\sum_{t=1}^T\Delta L^j(\tilde\theta, x_t, a_t^1, a_t^2, y_t)\bigg)\\
&\ge\EE_{S_T}\log\EE_{\tilde\theta\sim p_0}\exp\bigg(-\frac{L_0^2T\eta}{2}d(\tilde\theta, \theta_*)^2-L_0T(\eta+2\mu) d(\tilde\theta, \theta_*)\bigg)\\
&\ge\log\int_{\tilde\theta:d(\tilde\theta, \theta_*)\le\delta}p_0(\tilde\theta)-\frac{L_0^2T\eta\delta^2}{2}-L_0T(\eta+2\mu)\delta\\
&\ge\log p_0(\theta_*)-L\delta+\log\bar\mu(\{\theta\in\Theta:d(\theta, \theta_*)\le\delta\})-\frac{L_0^2T\eta\delta^2}{2}-L_0T(\eta+2\mu)\delta,
\end{align*}
where the first inequality holds due to \eqref{eq:exp_bound_nonlinear}, \eqref{eq:FG_bound_nonlinear} and \eqref{eq:I2_nonlinear}, the second inequality holds because $\{\theta\in\Theta:d(\theta, \theta_*)\le\delta\}\subset\Theta$, and the last inequality holds due to \eqref{eq:measure_nonlinear}. Taking $\delta=2/(L_0T)$, we have
\begin{align*}
Z_T^j&\ge\log p_0(\theta_*)+\log\bar\mu(\{\theta\in\Theta:d(\theta, \theta_*)\le2/(L_0T)\})-\frac{2L}{L_0T}-\frac{2\eta}{T}-2(\eta+2\mu)\\
&\ge\log p_0(\theta_*)+\log\bar\mu(\{\theta\in\Theta:d(\theta, \theta_*)\le2/(L_0T)\})-\frac{2L}{L_0T}-4(\eta+\mu),
\end{align*}
where the second inequality holds because $T\ge1$.
\end{proof}


\section{Proof of Lemma \ref{lemma:bound_LS-FG}}\label{subsection:bound_LS-FG}

\begin{proof}[Proof of Lemma \ref{lemma:bound_LS-FG}]
The difference of the potential between steps can be bounded as
\begin{align}
Z_t^j-Z_{t-1}^j&=\EE_{S_t}\log\frac{\EE_{\tilde\btheta\sim p_0}W_t^j(\tilde\btheta|S_t)}{\EE_{\tilde\btheta\sim p_0}W_{t-1}^j(\tilde\btheta|S_{t-1})}\nonumber\\
&=\EE_{S_t}\log\frac{\EE_{\tilde\btheta\sim p_0}[W_{t-1}^j(\btheta|S_{t-1})\exp(-\Delta L^j(\tilde\btheta, x_t, a_t^1, a_t^2, y_t))]}{\EE_{\tilde\btheta\sim p_0}W_{t-1}^j(\tilde\btheta|S_{t-1})}\nonumber\\
&=\EE_{S_t}\log\EE_{\tilde\btheta\sim p^j(\cdot|S_{t-1})}\exp(-\Delta L^j(\tilde\btheta, x_t, a_t^1, a_t^2, y_t))\nonumber\\
&\le\EE_{S_{t-1}, x_t, a_t^1, a_t^2}\log\EE_{\tilde\btheta\sim p^j(\cdot|S_{t-1})}\EE_{y_t|x_t, a_t^1, a_t^2}\exp(-\Delta L^j(\tilde\btheta, x_t, a_t^1, a_t^2)),\label{eq:Zt-Zt-1_bound0}
\end{align}
where the first equality holds due to the definition of $Z_t$ in \eqref{eq:def_Zt}, the second equality holds due to the definition of definition of $W_t$ in \eqref{eq:def_Wt}, the third equality holds due to \eqref{eq:link_posterior_Wt}, and the inequality holds due to Jensen's inequality. Note that
\begin{align}
&\EE_{y_t|x_t, a_t^1, a_t^2}\exp(-\Delta L(\tilde\btheta, x_t, a_t^1, a_t^2))=\exp(\mu\FG_t^j(\tilde\btheta))\nonumber\\
&\qquad\cdot\bigg[\bigg(\frac{1+\exp(-\langle\btheta_*, \bphi(x_t, a_t^1)-\bphi(x_t, a_t^2)\rangle)}{1+\exp(-\langle\tilde\btheta, \bphi(x_t, a_t^1)-\bphi(x_t, a_t^2)\rangle)}\bigg)^\eta\cdot\frac1{1+\exp(-\langle\btheta_*, \bphi(x_t, a_t^1)-\bphi(x_t, a_t^2)\rangle)}\nonumber\\
&\qquad+\bigg(\frac{1+\exp(\langle\btheta_*, \bphi(x_t, a_t^1)-\bphi(x_t, a_t^2)\rangle)}{1+\exp(\langle\tilde\btheta, \bphi(x_t, a_t^1)-\bphi(x_t, a_t^2)\rangle)}\bigg)^\eta\cdot\frac1{1+\exp(\langle\btheta_*, \bphi(x_t, a_t^1)-\bphi(x_t, a_t^2)\rangle)}\bigg]\nonumber\\
&=\exp(\mu\FG_t^j(\tilde\btheta)+\sigma(\langle(1-\eta)\btheta_*+\eta\tilde\btheta, \bphi(x_t, a_t^1)-\bphi(x_t, a_t^2)\rangle)\nonumber\\
&\qquad-(1-\eta)\sigma(\langle\btheta_*, \bphi(x_t, a_t^1)-\bphi(x_t, a_t^2)\rangle)-\eta\sigma(\langle\tilde\btheta, \bphi(x_t, a_t^1)-\bphi(x_t, a_t^2)\rangle))\nonumber\\
&\le\exp(\mu\FG_t^j(\tilde\btheta)-e^{-2B}/8\cdot\eta(1-\eta)\LS_t(\tilde\btheta))\nonumber\\
&\le\exp(\mu\FG_t^j(\tilde\btheta)-\eta e^{-2B}/16\cdot\LS_t(\tilde\btheta)),\label{eq:expectation_yt}
\end{align}
where the second equality holds due to the definition of $\Delta L^j$ in \eqref{eq:def_DeltaLt}, the first inequality holds due to Lemma \ref{lemma:sigma_second_diff}, and the second inequality holds because $\eta\le1/2$. Plugging \eqref{eq:expectation_yt} into \eqref{eq:Zt-Zt-1_bound0}, we have
\begin{align}
Z_t^j-Z_{t-1}^j&\le\EE_{S_{t-1}, x_t, a_t^1, a_t^2}\log\EE_{\tilde\btheta\sim p^j(\cdot|S_{t-1})}\exp(\mu\FG_t^j(\tilde\btheta)-\eta e^{-2B}/16\cdot\LS_t(\tilde\btheta))\nonumber\\
&\le\frac12\underbrace{\EE_{S_{t-1}, x_t, a_t^1, a_t^2}\log\EE_{\tilde\btheta\sim p^j(\cdot|S_{t-1})}\exp(2\mu\FG_t^j(\tilde\btheta))}_{I_1}\nonumber\\
&\qquad+\frac12\underbrace{\EE_{S_{t-1}, x_t, a_t^1, a_t^2}\log\EE_{\tilde\btheta\sim p^j(\cdot|S_{t-1})}\exp(-\eta e^{-2B}/8\cdot\LS_t(\tilde\btheta))}_{I_2},\label{eq:Zt-Zt-1_bound1}
\end{align}
where the second inequality holds due to Cauchy-Schwarz inequality. For $I_1$, note that $\FG_t^j(\tilde\btheta)\in[-4B, 4B]$, so by Lemma \ref{lemma:Hoeffding}, we have
\begin{align}\label{eq:FG_bound}
I_1\le2\mu\EE_{S_{t-1}, x_t, a_t^1, a_t^2}\EE_{\tilde\btheta\sim p^j(\cdot|S_{t-1})}\FG_t^j(\tilde\btheta)+64\mu^2B^2.
\end{align}
For $I_2$, we have
\begin{align}
I_2&\le\EE_{S_{t-1}, x_t, a_t^1, a_t^2}\log\EE_{\tilde\btheta\sim p^j(\cdot|S_{t-1})}[1-\eta e^{-2B}/8\cdot\LS_t(\tilde\btheta)+\eta^2e^{-4B}/128\cdot(\LS_t(\tilde\btheta))^2]\nonumber\\
&\le\EE_{S_{t-1}, x_t, a_t^1, a_t^2}\log\EE_{\tilde\btheta\sim p^j(\cdot|S_{t-1})}[1-\eta e^{-2B}/8\cdot\LS_t(\tilde\btheta)(1-\eta e^{-2B}/16\cdot(4B)^2)]\nonumber\\
&\le\EE_{S_{t-1}, x_t, a_t^1, a_t^2}\log\EE_{\tilde\btheta\sim p^j(\cdot|S_{t-1})}[1-\eta e^{-2B}/8\cdot\LS_t(\tilde\btheta)(1-1/2\cdot1/e^2)]\nonumber\\
&\le\EE_{S_{t-1}, x_t, a_t^1, a_t^2}\log\EE_{\tilde\btheta\sim p^j(\cdot|S_{t-1})}[1-\eta e^{-2B}/9\cdot\LS_t(\tilde\btheta)]\nonumber\\
&\le-\eta e^{-2B}/9\cdot\EE_{S_{t-1}, x_t, a_t^1, a_t^2}\EE_{\tilde\btheta\sim p^j(\cdot|S_{t-1})}\LS_t(\tilde\btheta),\label{eq:LS_bound}
\end{align}
where the first inequality holds because $e^z\le1+z+z^2/2$ for all $z\le0$, the second inequality holds because $\LS_t(\btheta)\le(4B)^2$, the second inequality holds because $Be^{-B}\le1/e$ for all $B>0$, the fourth inequality holds because $1/8\cdot(1-1/2\cdot1/e^2)\ge1/9$, and the last inequality holds because $\log(1+z)\le z$. Plugging \eqref{eq:FG_bound} and \eqref{eq:LS_bound} into \eqref{eq:Zt-Zt-1_bound1}, we have
\[
Z_t^j-Z_{t-1}^j\le32\mu^2B^2+\EE_{S_{t-1}, x_t, a_t^1, a_t^2}\EE_{\tilde\btheta\sim p^j(\cdot|S_{t-1})}[\mu\FG_t^j(\tilde\btheta)-\eta e^{-2B}/18\cdot\LS_t(\tilde\btheta)].
\]
Rearranging terms, we obtain
\[
\EE_{S_{t-1}, x_t, a_t^1, a_t^2}\EE_{\tilde\btheta\sim p^j(\cdot|S_{t-1})}\bigg[\frac{e^{-2B}\eta}{18\mu}\LS_t(\tilde\btheta)-\FG_t^j(\tilde\btheta)\bigg]\le\mu^{-1}(Z_{t-1}^j-Z_t^j)+32\mu B^2.
\]

\end{proof}

\section{Auxiliary Lemmas}

\begin{lemma}[Decoupling lemma, Lemma D.1 in \citet{fan2023the}]\label{lemma:decoupling}
Let $P$ be a joint distribution over two $\RR^d$ spaces. For any constant $\lambda>0$, we have
\[
\EE_{(\btheta, \bphi)\sim P}\langle\btheta, \bphi\rangle\le d\lambda+\frac1{4\lambda}\EE_{(\btheta, \bphi)\sim P}\EE_{(\btheta', \bphi')\sim P}\langle\btheta', \bphi\rangle^2.
\]
\end{lemma}

\begin{lemma}[Hoeffding's lemma]\label{lemma:Hoeffding}
Let $X$ be a random variable that is bounded by $a\le X\le b$. Then for any constant $\lambda$,
\[
\EE\exp(\lambda X)\le\exp(\lambda\EE X+\lambda^2(b-a)^2/8).
\]
\end{lemma}

\begin{lemma}\label{lemma:sigma_second_diff}
Let $\sigma(z)=\log(1+\exp(-z))$. Then for any $\eta\in(0, 1)$ and $p, q\in[-2B, 2B]$, we have
\[
\sigma((1-\eta)p+\eta q)-(1-\eta)\sigma(p)-\eta\sigma(q)\le-\frac{e^{-2B}}{8}\eta(1-\eta)(q-p)^2.
\]
\end{lemma}
\begin{proof}
Without loss of generality, we assume that $p\le q$. Otherwise, we substitute $(p, q, \eta)\gets(q, p, 1-\eta)$. By Lagrange‘s mean value theorem, there exists $\xi(\eta, p, q)\in[p, q]$ such that
\[
\eta\sigma'((1-\eta)p+\eta q)-\eta\sigma'(q)=-\eta(1-\eta)(q-p)\sigma''(\xi(\eta, p, q)).
\]
Note that for any $\xi\in[-2B, 2B]$, the second derivative $\sigma''(\xi)$ is bounded by
\[
\sigma''(\xi)=\frac{1}{e^\xi+2+e^{-\xi}}\ge\frac{e^{-2B}}{4},
\]
so
\[
\eta\sigma'((1-\eta)p+\eta q)-\eta\sigma'(q)\le-\frac{e^{-2B}\eta(1-\eta)}{4}(q-p).
\]
Taking integral w.r.t. $q$ on both sides, we have
\[
\sigma((1-\eta)p+\eta q)-(1-\eta)\sigma(p)-\eta\sigma(q)\le-\frac{e^{-2B}}{8}\eta(1-\eta)(q-p)^2.
\]
\end{proof}

\begin{lemma}\label{lemma:DKL}
    For any $\btheta, \btheta'\in\Theta$, we have
    \begin{align*}
    \DKL(\PP_{\btheta}||\PP_{\btheta'})\le\frac18\EE_{\btheta}\bigg[\sum_{t=1}^T\langle\btheta-\btheta', \bphi(x_t, a_t^1)-\bphi(x_t, a_t^2)\rangle^2\bigg].
    \end{align*}
\end{lemma}
\begin{proof}
    We define the shorthand notation
    \begin{align*}
    p_{\btheta, t}\coloneqq\PP_{\btheta}[y_t=1|x_t, a_t^1, a_t^2]=\frac{1}{1+\exp(\langle\btheta, \bphi(x_t, a_t^2)-\bphi(x_t, a_t^1)\rangle)},
    \end{align*}
    then by decomposition properties of the relative entropy, we have
    \begin{align}
    \DKL(\PP_{\btheta}||\PP_{\btheta'})&=\EE_{\btheta}\bigg[\sum_{t=1}^T\DKL(\Ber(p_{\btheta, t})||\Ber(p_{\btheta', t}))\bigg].\label{eq:KL_decomp}
    \end{align}
    Denote
    \begin{align*}
    u&=\langle\btheta, \bphi(x_t, a_t^1)-\bphi(x_t, a_t^2)\rangle,\\
    v&=\langle\btheta', \bphi(x_t, a_t^1)-\bphi(x_t, a_t^2)\rangle,
    \end{align*}
    then
    \begin{align}
    \DKL(\Ber(p_{\btheta, t})||\Ber(p_{\btheta', t}))&=\frac1{1+e^{-u}}\cdot\log\frac{1+e^{-v}}{1+e^{-u}}+\frac1{1+e^u}\cdot\log\frac{1+e^v}{1+e^u}\nonumber\\
    &=\log(1+e^{-v})-\log(1+e^{-u})-\frac{-1}{1+e^u}(v-u)\nonumber\\
    &=\frac{e^\xi}{(1+e^\xi)^2}\cdot\frac{(v-u)^2}2\nonumber\\
    &\le\frac{(v-u)^2}8,\label{eq:DKL_bound_single}
    \end{align}
    where the first equality holds due to definition of $p_{\btheta, t}$, the third equality holds due to Taylor expansion with Langragian remainder, and the inequality holds due to AM-GM inequality. Plugging \eqref{eq:DKL_bound_single} into \eqref{eq:KL_decomp}, we derive the desired upper bound for $\DKL(\PP_{\btheta}||\PP_{\btheta'})$.
\end{proof}

\bibliography{arxiv_v2}
\bibliographystyle{ims}

\end{document}